\PassOptionsToPackage{preprint}{neurips_2026}

\documentclass{article}
\usepackage{neurips_2026}

\usepackage[utf8]{inputenc}
\usepackage[T1]{fontenc}
\usepackage{courier}
\usepackage{microtype}
\usepackage{graphicx}
\usepackage{subcaption}
\usepackage{booktabs}
\usepackage{placeins}
\usepackage{amsmath}
\usepackage{amssymb}
\usepackage{amsfonts}
\usepackage{mathtools}
\usepackage{amsthm}
\usepackage{multirow}
\usepackage{colortbl}
\usepackage{algorithm}
\usepackage{algorithmic}
\usepackage{xcolor}
\usepackage{url}
\usepackage{hyperref}
\usepackage[capitalize,noabbrev]{cleveref}
\usepackage{wrapfig}
\usepackage{float}
\usepackage{enumitem}

\newcommand{\method}{\textsc{Pr}\textsuperscript{2}}
\newcommand{\methodthree}{\textsc{Pr}\textsuperscript{3}}
\newcommand{\rtwo}{\textsc{R}\textsuperscript{2}}
\newcommand{\rthree}{\textsc{R}\textsuperscript{3}}
\newcommand{\methodmath}{\mathrm{Pr}^{2}}

\theoremstyle{plain}
\newtheorem{theorem}{Theorem}[section]
\newtheorem{proposition}[theorem]{Proposition}
\newtheorem{lemma}[theorem]{Lemma}

\theoremstyle{definition}

\theoremstyle{remark}

\definecolor{TableHeader}{gray}{0.94}
\definecolor{TableHighlight}{RGB}{232,240,248}

\hypersetup{
    colorlinks,
    linkcolor={red!50!black},
    citecolor={blue!50!black},
    urlcolor={blue!80!black}
}
\title{\method: Predictive Routing Replay \\for MoE-Based LLM Reinforcement Learning}

\author{%
Daize Dong$^\clubsuit$,\quad
Junlin Chen$^\clubsuit$,\quad
Haolong Jia$^\clubsuit$,\quad
Jiang Liu$^\blacklozenge$,\\\bfseries
Jiawei Wu$^\clubsuit$,\quad
Huanwei Di$^\clubsuit$,\quad
Jialian Wu$^\blacklozenge$,\quad
Zhengzhong Liu$^\bigstar$,\\\bfseries
Zicheng Liu$^\blacklozenge$,\quad
Emad Barsoum$^\blacklozenge$,\quad
Dimitris N. Metaxas$^\clubsuit$,\quad
Hongyi Wang$^\clubsuit$\\
\normalsize $^\clubsuit$ Rutgers University \ \ $^\blacklozenge$ AMD \ \ $^\bigstar$ MBZUAI\\
\small \texttt{daize.dong@rutgers.edu; hw689@cs.rutgers.edu}
}

\begin{document}

\maketitle

\begin{abstract}
Mixture of Experts (MoE) Large Language Models (LLMs) achieve strong performance at scale. However, reinforcement learning (RL) on MoE-based LLMs often suffers from training instability. A root cause is \textit{router drift}, i.e., expert activations can change drastically across model updates and differ between disaggregated rollout and training phases, causing large rollout--training mismatch and unstable importance sampling weights in PPO-style RL algorithms.
Routing replay mitigates this issue by freezing the replay route within each reasoning trajectory, but it ignores how the router evolves under off-policy updates and thus causes \textit{router staleness}. To address this limitation, we propose \textbf{Predictive Routing Replay (\method)}, which augments each router with a lightweight evolution predictor that learns to anticipate short-horizon router evolution.
During the rollout phase, we use the predictive routing distribution to apply top-$k$ routing, enabling gradients to reach experts that are likely to become active after updates. During the training phase, we replay the resulting predicted route to retain consistency for stable importance estimation. Theoretical analysis and experiments support that \method~reduces routing-induced mismatch, improves RL stability, and yields stronger performance across various reasoning benchmarks.
\end{abstract}

\section{Introduction}
\label{sec:intro}
Large language models (LLMs) have demonstrated strong reasoning capabilities when scaled through increased model parameters and training compute~\citep{openai2024gpt4technicalreport,deepseekai2026deepseekv4,comanici2025gemini25pushingfrontier}. 
Mixture of Experts (MoE) models have recently emerged as a promising LLM architecture, enabling training compute to scale sublinearly with model size~\citep{jiang2024mixtral,zhu2024llama,liu2024deepseek,yang2025qwen3,agarwal2025gpt,team2025kimi,5team2026glm5vibecodingagentic,deepseekai2026deepseekv4}. 
This efficiency is achieved by activating a subset of expert networks within the large model, where the activated parameters typically account for less than $10\%$ of the total model parameters~\citep{liu2024deepseek,yang2025qwen3,5team2026glm5vibecodingagentic}. This conditional computation paradigm has enabled models with trillions of parameters with manageable inference costs, and has become a core design choice in large-scale LLMs.

Beyond pretraining, reinforcement learning (RL) has emerged as a central mechanism for optimizing pretrained LLMs to improve reasoning and agentic capabilities, as well as alignment with human preferences~\citep{ouyang2022training,rafailov2023direct,shao2024deepseekmath,liu2025k2,gao2025beyond}. Algorithms such as Proximal Policy Optimization (PPO)~\citep{schulman2017proximal} and its variants~\citep{shao2024deepseekmath,guo2025deepseek,yu2025dapo} are now widely used to improve reasoning, long-horizon decision making, and tool use tasks~\mbox{\citep{chen2022program,yao2023react}}.
However, applying RL to MoE-based LLMs introduces a unique and underexplored set of stability challenges that do not arise in dense models~\citep{yao2025offpolicy}.

A key difficulty for conducting stable RL on MoE-based LLMs stems from the presence of learned routers that dynamically assign tokens to experts~\citep{yao2025offpolicy,zheng2025stabilizing,kim2025defending}. When policy updates reuse trajectories generated by a stale snapshot, MoE models expose this off-policy gap as \textit{router drift}, where the same token may be routed to different experts by the old snapshot and the current training policy. Such routing changes alter the computation path behind PPO-style importance ratios and can amplify their variance, thereby destabilizing policy optimization~\citep{schulman2017proximal,shao2024deepseekmath,yu2025dapo}.

To mitigate routing mismatch, routing replay records expert routes before training updates and reuses them during gradient evaluation; the old-snapshot variant records routes from $\pi_{\theta_{\text{old}}}$~\citep{zheng2024sglang}.
While effective at restoring routing consistency, routing replay introduces its own limitations.
By fixing the cached routes, it prevents gradients from reaching experts that would become active after subsequent policy updates.
Over time, this induces \textit{router staleness}, as newly activated experts under the updated routing policy are excluded from gradient updates. Theoretically, this yields stale gradient estimation under replayed routes, where gradients are evaluated under a fixed replay route whose distribution deviates from the evolving current training policy.

In this work, we revisit routing replay from a principled perspective. We formalize routing replay as replacing the current routing distribution with a degenerate replay measure and derive a bound on router staleness, showing that replay staleness controls route-induced gradient deviation in fixed-route PPO gradients. This view reveals:
\begin{quote}
\emph{Stable importance estimation favors frozen routing, while effective learning requires routing distributions that track policy evolution.}
\end{quote}

\begin{figure}[t]
    \centering
    \vspace{-14pt}
    \includegraphics[width=1\linewidth]{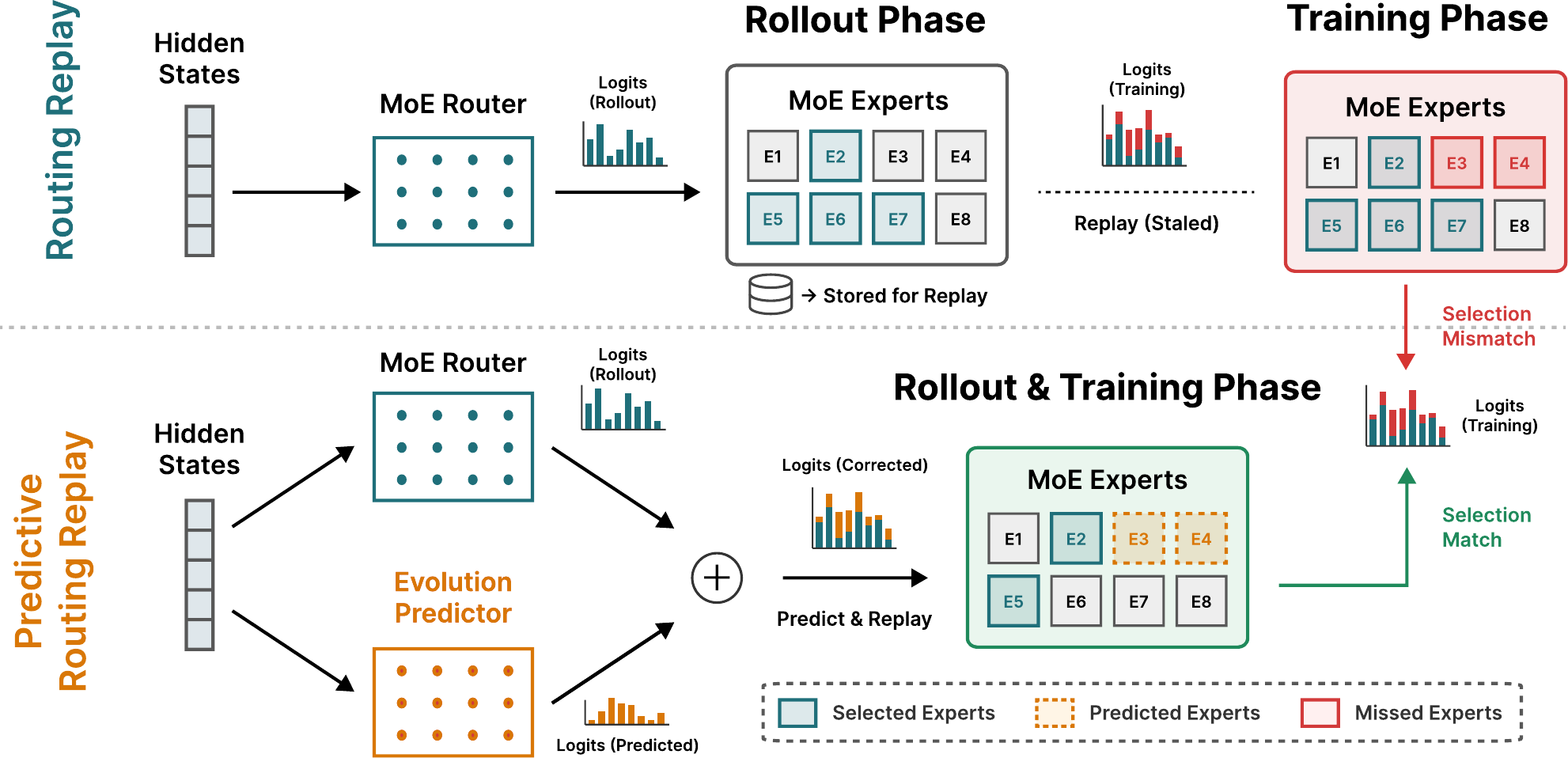}
    \vspace{-3pt}
    \caption{\textbf{Overview of Predictive Routing Replay (\method).} Routing replay stabilizes MoE RL by fixing routes, but cached routes become stale after a few off-policy steps. \method~adds an evolution predictor before top-$k$ selection, caches the predicted expert indices during rollout, and replays them during training to preserve route consistency while tracking short-horizon router evolution.}
    \label{fig:head}
    \vspace{-12pt}
\end{figure}
Motivated by this perspective, we propose \textbf{Predictive Routing Replay (\method)}, which augments routing replay with a route prediction scheme (Figure~\ref{fig:head}). \method~augments the router in each MoE layer with a lightweight \textit{evolution predictor} that anticipates \textit{router drift}. During the rollout phase, the predictor outputs a learned logit bias to adjust the routing distribution. We then perform top-$k$ routing under the biased logits to obtain a predicted expert index, which is cached alongside trajectories for subsequent updates. During the training phase, we replay the cached top-$k$ expert indices while disabling the predictor, keeping the replay route fixed to stabilize PPO-style importance ratios yet allowing gradient flow to experts that are likely to become active after policy updates. We train the evolution predictor with a KL divergence objective motivated by this bound on router staleness, along with an efficient training scheme and a dedicated learning-rate multiplier for fast adaptation.

\method~can be integrated into existing MoE RL frameworks, e.g., VeRL~\citep{sheng2025hybridflow}, Slime~\citep{slime_github} and TRL~\citep{vonwerra2020trl}, without modifying the policy optimization objective. Empirically, \method~substantially reduces routing mismatch, stabilizes RL training, and improves performance over strong baselines~\citep{guo2025deepseek,yu2025dapo,zheng2025group} across reasoning tasks with negligible computation overhead. Notably, on AIME24~\citep{aime}, GRPO with \method\ achieves $40.31\%$ accuracy on Qwen3-30B-A3B-Base~\citep{yang2025qwen3}, improving over routing replay and GSPO by $12.29\%$ and $9.38\%$ points, respectively.

Our main contributions are summarized below.
\begin{itemize}[leftmargin=1.5em]
\vspace{-4pt}
    \item We formalize \textit{router staleness} as a key source of degradation in routing replay for MoE-based LLM reinforcement learning, and derive a divergence-based bound showing that replay staleness controls route-induced gradient deviation in fixed-route PPO gradients.
    \vspace{-2pt}
    \item We propose \textbf{Predictive Routing Replay (\method)}, which predicts short-horizon router evolution using evolution predictors trained with a KL objective motivated by the above bound, enabling rollout-time routing to anticipate future expert activation.
    \item We demonstrate \method~significantly reduces routing mismatch, improves RL stability, and yields stronger performance across several reasoning tasks, establishing it as an effective alternative.
\end{itemize}

\section{Related Work}
\label{sec:related}

\paragraph{Mixture of Experts.} The Mixture of Experts (MoE) layer scales model capacity through conditional computation, where a learned router selects a subset of expert networks per token to reduce FLOPs while retaining a large number of parameters~\citep{shazeer2017outrageously,lepikhin2020gshard,fedus2022switch,zhu2024llama,jiang2024mixtral,yang2025qwen3,team2025kimi}. 
A large body of work studies the routing strategy~\citep{lewis2021base,huang2024harder}, load balancing~\citep{shazeer2017outrageously,fedus2022switch,liu2024deepseek}, and gradient estimation~\citep{kool2021unbiased,liu2023sparsebackpropagationmoetraining,liu2024gringradientinformedmoe} in MoE training.
Analysis also highlights that routing decisions can be fragile~\citep{dai2022stablemoe,zoph2022st}, which leads to unstable gradient flow and optimization during training~\citep{kim2025defending}. These findings motivate treating router evolution as an explicit modeling target when systems separate rollout from training~\citep{yao2025offpolicy}.

\paragraph{Reinforcement Learning in LLMs.}
Reinforcement learning (RL) is widely used to align LLMs for improved reasoning and instruction following, with PPO-style algorithms and their variants serving as widely used optimization methods~\citep{ouyang2022training,shao2024deepseekmath,guo2025deepseek,yu2025dapo,zhao2026geometricmean}.
Large-scale LLM RL is often implemented in disaggregated pipelines, where rollouts and training are lagged, thereby introducing off-policy effects ~\citep{mnih2016asynchronous,espeholt2018impala}. 
For MoE-based LLMs, this mismatch is exacerbated by \textit{router drift}, where the same token can be assigned to different experts during rollout and training forward passes. 
This routing mismatch can amplify importance-ratio variance and destabilize policy optimization, harming RL performance~\citep{yao2025offpolicy}.

\paragraph{Stabilizing RL in MoE.}
To overcome \textit{router drift}, routing replay caches expert indices and reuses them during training to ensure routing consistency; the old-snapshot variant~\citep{zheng2024sglang} records routes from the old policy snapshot. Another variant, rollout routing replay~\citep{ma2025stabilizing}, targets the rollout-engine setting.
Beyond routing replay, GSPO~\citep{zheng2025group} defines the importance ratio based on sequence likelihood and performs sequence-level clipping, avoiding token-level importance discrepancies induced by \textit{router drift}.
Other methods improve stability by modifying the clipping behavior~\citep{gao2025soft} or directly adjusting importance ratios to reduce sensitivity to off-policy errors~\citep{zhang2025towards}.
In contrast, our work uses a staleness-control view of routing replay to predict short-horizon router evolution.

\section{Preliminaries}
\label{sec:preliminary}

We study off-policy reinforcement learning of an autoregressive MoE-based LLM. Following the standard splitting of rollout and training stages in PPO-style optimization, we let an old policy snapshot generate trajectories that are then reused to update the current training policy.

\subsection{Off-Policy Reinforcement Learning}
Let $\theta$ denote the current training parameters and $\theta_{\text{old}}$ the stale parameters associated with the rollout batch. Given a prompt $x$, the old policy snapshot $\pi_{\theta_{\text{old}}}$ samples a completion $y=(y_1,\dots,y_T)$ one token at a time. At step $t$, the next token $y_t$ is generated conditioned on the prefix $(x,y_{<t})$, and a reward $R(x,y)$ is assigned after generation. We use $\pi_\theta$ to denote the current training policy. This gives the off-policy objective
\begin{align*}
    J(\theta)
    =
    \mathbb{E}_{x,\,y\sim \pi_{\theta_{\text{old}}}}
    \left[
    w_\theta(x,y)\,R(x,y)
    \right].
    \label{eq:offpolicy_obj}
\end{align*}
PPO-style methods~\citep{ouyang2022training,shao2024deepseekmath,guo2025deepseek} typically factorize importance weights $w_\theta$ into token-level importance ratios $r_t(\theta)$.
\begin{equation}
    w_\theta(x,y)=\prod_{t=1}^{T} r_t(\theta),
    \quad
    r_t(\theta)=\frac{\pi_\theta(y_t\mid x,y_{<t})}{\pi_{\theta_{\text{old}}}(y_t\mid x,y_{<t})}.
    \label{eq:token_ratio}
\end{equation}
We focus on PPO-style policy optimization, where training is driven by importance ratios in Eq.~\eqref{eq:token_ratio}. We summarize the detailed objectives in Appendix~\ref{app:ppo}.

\subsection{Reinforcement Learning on Mixture of Experts}
We consider an MoE-based LLM with $L$ MoE layers and $N$ experts per layer. At token $t$ and layer $l$, the router produces logits and routing probabilities over experts, and selects a top-$k$ expert index $\mathcal{I}_t^{(l)}\subseteq\{1,\dots,N\}$. We  denote the layer-wise route at step $t$ as
\begin{align*}
    \mathcal{R}_t := \{\mathcal{I}_t^{(1)},\ldots,\mathcal{I}_t^{(L)}\}.
    \label{eq:routing_path}
\end{align*}
Let $\rho^{\pi}_{\theta,t}(\mathcal{R}_t\mid x,y_{<t})$ denote the current route distribution induced by the routers of $\pi_\theta$ at token $t$ (possibly degenerate under deterministic top-$k$ selection). Conditioned on a route, the token distribution is $\pi_\theta(y_t\mid x,y_{<t},\mathcal{R}_t)$, and the marginal token distribution satisfies
\begin{equation}
    \pi_\theta(y_t\mid x,y_{<t})
    =
    \mathbb{E}_{\mathcal{R}_t\sim \rho^{\pi}_{\theta,t}(\cdot\mid x,y_{<t})}
    \big[\pi_\theta(y_t\mid x,y_{<t},\mathcal{R}_t)\big].
    \label{eq:moe_marginal}
\end{equation}
In an autoregressive MoE model, computing $\pi_\theta(y_t\mid x,y_{<t},\mathcal{R}_t)$ also depends on the key--value caches of all previous positions, which were produced under the routes selected at those steps. Eq.~\eqref{eq:moe_marginal} therefore conditions implicitly on the entire sequence of routes $\mathcal{R}_{\le t}$ used along the prefix, and we suppress this conditioning for notational brevity.

\paragraph{Router Drift.}
Router drift acts as a form of the stale policy gap specific to MoE. The same prefix $(x,y_{<t})$ may induce different routes under the old snapshot and current training policy. Let $\mathcal{R}^{\pi}_{\theta_{\text{old}},t}$ denote the route induced by $\pi_{\theta_{\text{old}}}$ and let $\mathcal{R}^{\pi}_{\theta,t}$ denote the route induced by $\pi_\theta$, with corresponding routing distributions $\rho^{\pi}_{\theta_{\text{old}},t}(\cdot\mid x,y_{<t})$ and $\rho^{\pi}_{\theta,t}(\cdot\mid x,y_{<t})$, respectively.
Router drift refers to the discrepancy $\mathcal{R}^{\pi}_{\theta_{\text{old}},t}\neq \mathcal{R}^{\pi}_{\theta,t}$, which yields inconsistent expert activation for tokens and can amplify the variance of importance ratios in PPO-style optimization~\citep{zheng2025stabilizing,yao2025offpolicy}.

\paragraph{Routing Replay.}
Routing replay fixes a cached replay route $\tilde{\mathcal{R}}_t$ during training to enforce routing consistency, denoted as:
\begin{equation}
\text{replay forward at step }t \,=\, \pi_\theta(y_t\mid x,y_{<t},\tilde{\mathcal{R}}_t),
\label{eq:replay_measure}
\end{equation}
where $\tilde{\mathcal{R}}_t$ is a cached replay route.
Routing replay sets $\tilde{\mathcal{R}}_t=\mathcal{R}^{\pi}_{\theta_{\text{old}},t}$ by replaying the route induced by the old snapshot~\citep{zheng2024sglang}.
This method stabilizes importance estimation by fixing the route during gradient evaluation, but may prevent gradients from reaching experts that would become active under subsequent router updates.

\subsection{Router Staleness}
\label{sec:router_staleness}

The stability induced by routing replay comes at the cost of staleness. As the training router evolves, a cached route may drift away from the route preferred by the current training policy. Let $\rho_t^{\mathrm{rep}}$ denote the route distribution induced by a replay scheme. For deterministic replay, $\rho_t^{\mathrm{rep}}$ is a point mass at the cached route. The TV distance is therefore binary on raw deterministic routes, and we use it here as a discrete-route summary. Appendix~\ref{app:prop-proof} introduces a soft layer-wise categorical relaxation on which the same divergence is differentiable and is used by the \method~predictive loss. We define token-level router staleness as
\begin{align*}
\mathcal{S}_t(\rho_t^{\mathrm{rep}})
=
D_{\mathrm{TV}}\!\left(
\rho^{\pi}_{\theta,t}(\cdot\mid x,y_{<t}),
\rho_t^{\mathrm{rep}}(\cdot\mid x,y_{<t})
\right).
\end{align*}
Let $g_t(P_t,\theta)$ denote the fixed-route PPO gradient kernel averaged under a route distribution $P_t$. Appendix~\ref{app:bias} shows that, for any bounded fixed-route gradient kernel, replay staleness controls the route-induced gradient deviation.
\begin{align*}
\left\|g_t(\rho^{\pi}_{\theta,t},\theta)-g_t(\rho_t^{\mathrm{rep}},\theta)\right\|
\le
2M_t\,\mathcal{S}_t(\rho_t^{\mathrm{rep}}).
\end{align*}
Thus, frozen replay is stable only when cached and current routes remain close.

\section{\method: Predictive Routing Replay}
\label{sec:method}
The design principle of \method~is to keep deterministic replay consistency while predicting the expert index likely to become active after short-horizon policy updates. \method~keeps the fixed-route pattern of routing replay but replaces the cached route with a predicted one. For clarity, we describe \method~relative to routing replay. The same prediction mechanism can be inserted before any deterministic replay route. For each token, \method~constructs predicted expert indices from route-recording features of the old snapshot $\pi_{\theta_{\text{old}}}$ and trains a lightweight evolution predictor against routing distributions observed later on the same cached batch. Figure~\ref{fig:method} summarizes our core idea.

\begin{figure}[t]
\vspace{-8pt}
\centering
\includegraphics[width=0.68\linewidth]{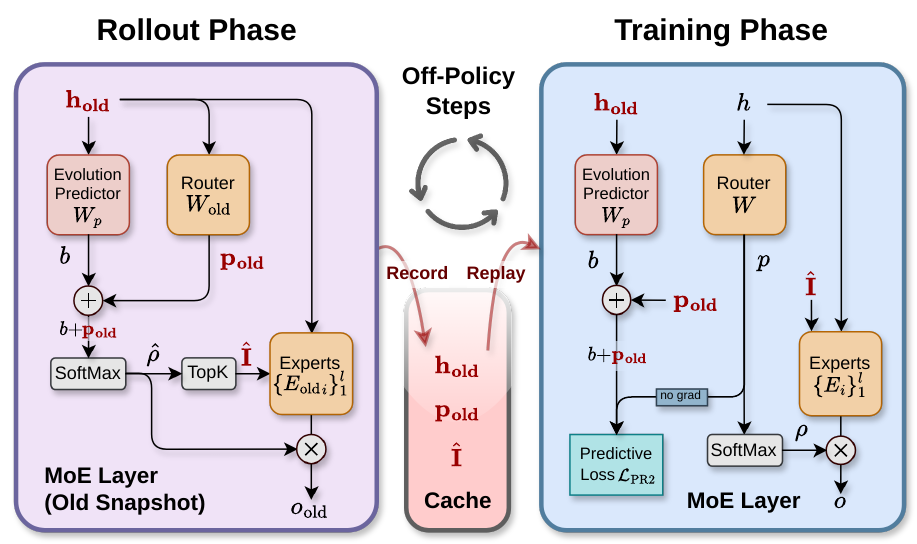}
\caption{\textbf{Detailed \method~workflow.} \textit{Left: rollout phase $\pi_{\theta_{\text{old}}}$.} \method~adds a learned logit bias before top-$k$ selection and caches the predicted expert indices $\hat{\mathcal{I}}$ and route-recording features $h_{\text{old}},p_{\text{old}}$. \textit{Right: training phase $\pi_\theta$.} Cached expert indices are replayed for route consistency. The evolution predictor is updated using cached features and the current routing distribution $\rho$ via the predictive loss $\mathcal{L}_{\mathrm{\method}}$.}
\label{fig:method}
\vspace{-12pt}
\end{figure}

\subsection{Predictive Replay Route}
\label{sec:pr2_method}

We first define how \method~predicts the replay route and how that route is reused during training.
The subscript ``old'' denotes route-recording quantities under the old policy snapshot
$\pi_{\theta_{\mathrm{old}}}$.

\paragraph{Route Prediction During Rollout.}
For each MoE layer $l$, let $h_{\mathrm{old},t}^{(l)}$ be the router input at token $t$ under
$\pi_{\theta_{\mathrm{old}}}$, and let
\begin{align*}
p_{\mathrm{old},t}^{(l)}
=
h_{\mathrm{old},t}^{(l)} W_{\mathrm{old}}^{(l)}
\end{align*}
be the corresponding router logits. Since the replay route must be fixed when data are recorded,
only these features are available at rollout time. \method~therefore introduces a lightweight
evolution predictor $W_p^{(l)}$, initialized at $\mathbf{0}$, which outputs an additive logit bias
\begin{align*}
b_t^{(l)}
=
h_{\mathrm{old},t}^{(l)} W_p^{(l)} .
\end{align*}
The corrected logits define the predictive routing distribution
\begin{equation}
\hat{\rho}_t^{(l)}
=
\mathrm{Softmax}\!\left(
p_{\mathrm{old},t}^{(l)} + b_t^{(l)}
\right).
\label{eq:pr2_predict}
\end{equation}
The predicted expert index $\hat{\mathcal{I}}_t^{(l)}$ is chosen by top-$k$ selection, forming
the predicted route $\hat{\mathcal{R}}_t$:
\begin{equation}
\begin{split}
\hat{\mathcal{I}}_t^{(l)}
&=
\mathrm{TopK}\!\left(\hat{\rho}_t^{(l)}, k\right), \\
\hat{\mathcal{R}}_t
&=
\left\{\hat{\mathcal{I}}_t^{(l)}\right\}_{l=1}^{L}.
\end{split}
\label{eq:pr2_route}
\end{equation}
Using this predicted expert index, the MoE output during route recording is
\begin{equation}
o_{\mathrm{old},t}^{(l)}
=
\sum_{j\in \hat{\mathcal{I}}_t^{(l)}}
\hat{\rho}_{t,j}^{(l)}
E_{\mathrm{old},j}^{(l)}
\!\left(h_{\mathrm{old},t}^{(l)}\right).
\label{eq:pr2_old_output}
\end{equation}
The same predicted route is then cached for replay, i.e.,
$\tilde{\mathcal{R}}_t=\hat{\mathcal{R}}_t$ in Eq.~\eqref{eq:replay_measure}.
Thus, \method~keeps the fixed replay pattern of routing replay, but replaces the expert index from old snapshot with a predicted expert index. Weighting with $\hat{\rho}_{t}^{(l)}$ is intentional:
the predictor is trained to match the current router distribution $\rho_t^{(l)}$, and
$\hat{\rho}_{t}^{(l)}$ approximates the weights that $\pi_\theta$ would assign on the same index.

\begin{wrapfigure}[18]{r}{0.52\columnwidth}
\vspace{-18pt}
\footnotesize
\vspace{8pt}

\hrule height 0.6pt
\vspace{1pt}

\refstepcounter{algorithm}
\noindent\textbf{Algorithm~\thealgorithm} Predictive Routing Replay (\method)
\label{alg:pr2}

\vspace{2pt}
\hrule height 0.4pt
\vspace{3pt}

\begin{algorithmic}[1]
    \STATE \textbf{Route Prediction During Rollout} on old snapshot $\theta_{\mathrm{old}}$.
    \FOR{each token $t$ and MoE layer $l$}
        \STATE Compute $p_{\mathrm{old},t}^{(l)} \gets h_{\mathrm{old},t}^{(l)}W_{\mathrm{old}}^{(l)}$.
        \STATE Compute $b_t^{(l)} \gets h_{\mathrm{old},t}^{(l)}W_p^{(l)}$.
        \STATE Set $\hat{\rho}_t^{(l)} \gets
        \mathrm{Softmax}(p_{\mathrm{old},t}^{(l)} + b_t^{(l)})$.
        \STATE Select $\hat{\mathcal{I}}_t^{(l)} \gets
        \mathrm{TopK}(\hat{\rho}_t^{(l)}, k)$ and dispatch to experts.
        \STATE Cache $\hat{\mathcal{I}}_t^{(l)}$ and
        $(h_{\mathrm{old},t}^{(l)}, p_{\mathrm{old},t}^{(l)})$.
    \ENDFOR

    \STATE \textbf{Replay During Training} on training model $\theta$.
    \FOR{each inner RL update}
        \FOR{each token $t$ and layer $l$}
            \STATE Compute $\rho_t^{(l)} \gets \mathrm{Softmax}(p_t^{(l)})$.
            \STATE Replay $\hat{\mathcal{I}}_t^{(l)}$ as the MoE expert index.
        \ENDFOR
        \STATE Calculate $\mathcal{L}_{\methodmath}$ and update $\{W_p^{(l)}\}_{l=1}^{L}$.
        \STATE Calculate the base RL loss and update $\theta$.
    \ENDFOR
\end{algorithmic}

\vspace{3pt}
\hrule height 0.6pt
\end{wrapfigure}
\vspace{-10pt}

\paragraph{Replay During Training.}
For each MoE layer $l$, during training, the current policy $\pi_\theta$ computes router logits
$p_t^{(l)}$ and the corresponding routing distribution
\begin{align*}
\rho_t^{(l)}
=
\mathrm{Softmax}\!\left(p_t^{(l)}\right).
\end{align*}
Similar to routing replay, \method~does not run fresh top-$k$ selection. Instead, it reuses the
cached expert index $\hat{\mathcal{I}}_t^{(l)}$ and restricts MoE computation to that set. The
layer output is
\begin{equation}
o_t^{(l)}
=
\sum_{j\in \hat{\mathcal{I}}_t^{(l)}}
\rho_{t,j}^{(l)}
E_j^{(l)}\!\left(h_t^{(l)}\right).
\label{eq:pr2_output}
\end{equation}
This design freezes only the indices of selected experts during MoE computation, while expert outputs and policy gradients are always evaluated with the current parameters $\theta$. The predictor is bypassed in the training pass and supervised separately by the predictive loss. An analogous \methodthree~variant is described in Appendix~\ref{app:r3_pr3}.

\subsection{Evolution Predictor Training}
\label{sec:train_strategy}

The evolution predictor is trained on cached route-recording features to match the routing
preference of the current training router on the same rollout batch.

\vspace{-2mm}
\paragraph{Predictive Loss.}
Given a cached token $t$ and layer $l$, we reconstruct the predictive routing distribution from
the route-recording features using Eq.~\eqref{eq:pr2_predict}, and compare it with the current
routing distribution $\rho_t^{(l)}$:
\begin{equation}
\mathcal{L}_{\methodmath}
=
\sum_{l=1}^{L}
\mathbb{E}_t
\left[
D_{\mathrm{KL}}
\left(
\left\langle \rho_t^{(l)} \right\rangle
\,\middle\|\,
\hat{\rho}_t^{(l)}
\right)
\right],
\label{eq:pr2_loss}
\end{equation}
where $\langle\cdot\rangle$ denotes the stop-gradient operator. Gradients update only the
predictor parameters $\{W_p^{(l)}\}_{l=1}^{L}$. The current router acts as a teacher but is not
changed by this auxiliary loss. We use Eq.~\eqref{eq:pr2_loss} as the default \method~predictive
loss and examine a delta-matching alternative in Appendix~\ref{app:additional_exp_analysis}.
The cached index for the current batch remains fixed during training, and the updated predictor
takes effect at the next route-recording phase.

\vspace{-2mm}
\paragraph{Predictive Loss as a Staleness Surrogate.}
The staleness view in Section~\ref{sec:router_staleness} suggests choosing cached indices that
remain close to routes preferred by the current router. The predictive loss bounds the route-induced
gradient deviation:
\begin{align*}
\mathbb{E}_t
\!\left[
\left\|
g_t(P_t^{\rho},\theta)
-
g_t(P_t^{\hat{\rho}},\theta)
\right\|
\right]
\le
M\sqrt{2\,\mathcal{L}_{\methodmath}},
\end{align*}
where $P_t^{\rho}$ and $P_t^{\hat{\rho}}$ are the current and predictive soft route distributions.
Appendix~\ref{app:prop-proof} further shows that the same KL objective controls the
current-router mass regret of the predicted hard top-$k$ indices.

\vspace{-2mm}
\paragraph{Off-Policy Training and Learning-Rate Scaling.}
The \method~predictive loss is evaluated after the first inner update, when the current router has
moved away from the route-recording snapshot and provides a nontrivial target. The evolution
predictors use a dedicated learning-rate multiplier $\alpha$, allowing them to track router
evolution on the timescale of RL updates without changing the base PPO objective. We provide
the per-token rollout and training pseudocode in Algorithm~\ref{alg:pr2}; the full version is
deferred to Appendix~\ref{app:pr2_algo}.

\section{Experiments}
\label{sec:exps}

We organize our experiments around three empirical questions: whether \method~improves downstream reasoning under repeated rollout reuse, whether it stabilizes PPO-style optimization, and whether it keeps replay routes closer to the evolving router. We first describe the shared experimental setup, and then evaluate these questions through downstream accuracy, training dynamics, and routing analysis.

\vspace{-2mm}
\subsection{Implementation Details}
\paragraph{Models and Training Data.}
We evaluate \method in three MoE-based LLM settings. Qwen3-30B-A3B-Base~\citep{yang2025qwen3} is trained on DAPO-17K~\citep{yu2025dapo}, Moonlight-16B-A3B~\citep{liu2025muonscalablellmtraining} is trained on GSM8K~\citep{cobbe2021gsm8k}, and OLMoE-1B-7B~\citep{muennighoff2024olmoe} is trained on the RLVR-GSM dataset\footnote{\url{https://huggingface.co/datasets/allenai/RLVR-GSM}} following its original implementation. All methods are implemented in the VeRL~\citep{sheng2025hybridflow} framework.

\vspace{-2mm}
\paragraph{Baselines and Off-Policy Strength.}
We implement \method on top of GRPO~\citep{shao2024deepseekmath}, and compare it with GSPO~\citep{zheng2025group} and routing replay~\citep{zheng2024sglang} under the same settings. Rollout routing replay~\citep{ma2025stabilizing} is compared with \methodthree in Appendix~\ref{app:additional_main_results}. We denote each off-policy setting as off-$\kappa$, where $\kappa=\frac{\mathcal{B}_{\text{global}}}{\mathcal{B}_{\text{update}}}$ is the ratio between the rollout batch size and the training update batch size. A larger $\kappa$ means that each rollout batch is reused for more training updates, corresponding to stronger off-policy reuse.

\vspace{-2mm}
\paragraph{Evaluation Benchmarks.}
For the main Qwen3-30B-A3B-Base setting, we evaluate off-2, off-4, and off-8 performance on AIME24 and AIME25~\citep{aime}, AMC23~\citep{MAA_AMC}, and HMMT25~\citep{balunovic_srimatharena_2025}.
For Moonlight-16B-A3B and OLMoE-1B-7B, we report GSM8K and MATH500~\citep{hendrycks2021measuring} results in the cross-model evaluation. Appendix~\ref{app:train_detail} provides detailed implementation settings.

\begin{table*}[ht]
    \centering
    \begingroup
    \scriptsize
    \setlength{\tabcolsep}{6pt}
    \renewcommand{\arraystretch}{1.08}
    \resizebox{0.8\linewidth}{!}{%
    \begin{tabular}{llccccc}
        \toprule
        \rowcolor{TableHeader}\begin{tabular}[c] {@{}l@{}}\textbf{Policy}\end{tabular}
        & \begin{tabular}[c]{@{}l@{}}\textbf{Method}\end{tabular}
        & \begin{tabular}[c]{@{}c@{}}\textbf{AIME24}\\\textnormal{(Avg@32)}\end{tabular}
        & \begin{tabular}[c]{@{}c@{}}\textbf{AIME25}\\\textnormal{(Avg@32)}\end{tabular}
        & \begin{tabular}[c]{@{}c@{}}\textbf{AMC23}\\\textnormal{(Avg@16)}\end{tabular}
        & \begin{tabular}[c]{@{}c@{}}\textbf{HMMT25}\\\textnormal{(Avg@16)}\end{tabular}
        & \begin{tabular}[c]{@{}c@{}}\textbf{Average}\end{tabular} \\
        \midrule
        \multirow{4}{*}{Off-2}
        & GRPO          & 31.04 & 24.68 & 74.06 & 9.16  & 34.74 \\
        & GSPO          & 30.42 & 23.54 & 77.18 & 11.88 & 35.76 \\
        & GRPO + \rtwo~    & 35.73 & 25.73 & 77.50 & 11.25 & 37.55 \\
        & \cellcolor{TableHighlight}\textbf{GRPO + \method} & \cellcolor{TableHighlight}\textbf{47.71} & \cellcolor{TableHighlight}\textbf{32.81} & \cellcolor{TableHighlight}\textbf{87.50} & \cellcolor{TableHighlight}\textbf{19.17} & \cellcolor{TableHighlight}\textbf{46.80} \\
        \midrule
        \multirow{4}{*}{Off-4}
        & GRPO          & 25.42 & 17.71 & 70.93 & 8.33  & 30.60 \\
        & GSPO          & 32.50 & 22.08 & 79.38 & 12.08 & 36.51 \\
        & GRPO + \rtwo~    & 28.13 & 21.25 & 73.59 & 10.00 & 33.24 \\
        & \cellcolor{TableHighlight}\textbf{GRPO + \method} & \cellcolor{TableHighlight}\textbf{47.40} & \cellcolor{TableHighlight}\textbf{31.67} & \cellcolor{TableHighlight}\textbf{85.47} & \cellcolor{TableHighlight}\textbf{20.42} & \cellcolor{TableHighlight}\textbf{46.24} \\
        \midrule
        \multirow{4}{*}{Off-8}
        & GRPO          & 25.00 & 15.93 & 72.81 & 7.50  & 30.31 \\
        & GSPO          & 30.93 & 22.18 & 72.03 & 7.50  & 33.16 \\
        & GRPO + \rtwo~    & 28.02 & 21.77 & 76.88 & 8.30  & 33.74 \\
        & \cellcolor{TableHighlight}\textbf{GRPO + \method} & \cellcolor{TableHighlight}\textbf{40.31} & \cellcolor{TableHighlight}\textbf{28.54} & \cellcolor{TableHighlight}\textbf{83.13} & \cellcolor{TableHighlight}\textbf{15.63} & \cellcolor{TableHighlight}\textbf{41.90} \\
        \bottomrule
    \end{tabular}
    }
    \endgroup
    \caption{\textbf{Downstream reasoning accuracy on Qwen3-30B-A3B-Base.} \rtwo~refers to routing replay for simplicity. Bold values mark the best accuracy in each metric column within an off-policy strength, and shaded rows mark \method.}
    \label{tab:main_downstream_qwen}
    \vspace{2pt}
\end{table*}

\vspace{-2mm}
\subsection{Downstream Reasoning Accuracy}
\paragraph{Main Comparison.}
Table~\ref{tab:main_downstream_qwen} reports the main comparison on Qwen3-30B-A3B-Base under off-2, off-4, and off-8. We compare \method with GRPO, GSPO, and routing replay. \method~achieves the best average accuracy in all three settings, reaching $46.80\%$ under off-2, $46.24\%$ under off-4, and $41.90\%$ under off-8. Compared with routing replay, \method~improves average accuracy by $9.25\%$, $13.00\%$, and $8.16\%$ points, respectively. Compared with GSPO, the corresponding gains are $11.04\%$, $9.73\%$, and $8.74\%$ points.
The gains remain substantial as rollout reuse becomes stronger, showing that \method~is not limited to mild off-policy settings. This trend is consistent with the router-staleness view. Routing replay preserves route consistency, but its cached indices cannot adapt to the routes preferred by the evolving router after repeated updates. In contrast, \method~moves the replay indices toward the router's short-horizon evolution while retaining deterministic replay during training.

\begin{table}[!t]
\centering
\scriptsize
\setlength{\tabcolsep}{6pt}
\renewcommand{\arraystretch}{1.08}
\resizebox{0.88\linewidth}{!}{%
\begin{tabular}{llccc!{\color{black!35}\vrule width 0.35pt}ccc}
\toprule
\rowcolor{TableHeader}
\multicolumn{2}{c}{} &
\multicolumn{3}{c!{\color{black!35}\vrule width 0.35pt}}{\textbf{Moonlight-16B-A3B}} &
\multicolumn{3}{c}{\textbf{OLMoE-1B-7B}} \\
\rowcolor{TableHeader}
\textbf{Policy} &
\textbf{Method} &
\begin{tabular}[c]{@{}c@{}}\textbf{GSM8K}\\\textnormal{\scriptsize(Avg@1)}\end{tabular} &
\begin{tabular}[c]{@{}c@{}}\textbf{MATH500}\\\textnormal{\scriptsize(Pass@4)}\end{tabular} &
\textbf{Average} &
\begin{tabular}[c]{@{}c@{}}\textbf{GSM8K}\\\textnormal{\scriptsize(Avg@1)}\end{tabular} &
\begin{tabular}[c]{@{}c@{}}\textbf{MATH500}\\\textnormal{\scriptsize(Avg@4)}\end{tabular} &
\textbf{Average} \\
\midrule
\multirow{4}{*}{Off-2}
& GRPO          & 81.80 & 41.60 & 61.70 & \textbf{73.76} & 20.90 & 47.33 \\
& GSPO          & 81.50 & 51.00 & 66.25 & 73.16 & 20.95 & 47.06 \\
& GRPO + \rtwo~ & 82.79 & 52.60 & 67.69 & 72.56 & 20.80 & 46.68 \\
\rowcolor{TableHighlight}
& \textbf{GRPO + \method} & \textbf{83.17} & \textbf{54.60} & \textbf{68.88} & 73.47 & \textbf{21.45} & \textbf{47.46} \\
\midrule
\multirow{4}{*}{Off-4}
& GRPO          & 62.09 & 36.40 & 49.25 & 72.48 & 20.15 & 46.32 \\
& GSPO          & 62.85 & 37.80 & 50.33 & \textbf{73.09} & 22.05 & 47.57 \\
& GRPO + \rtwo~ & 73.69 & 49.20 & 61.45 & 72.71 & 21.85 & 47.28 \\
\rowcolor{TableHighlight}
& \textbf{GRPO + \method} & \textbf{74.07} & \textbf{51.20} & \textbf{62.64} & 72.78 & \textbf{23.00} & \textbf{47.89} \\
\midrule
\multirow{4}{*}{Off-8}
& GRPO          & 58.07 & 39.60 & 48.84 & 70.66 & 20.80 & 45.73 \\
& GSPO          & 62.78 & 38.80 & 50.79 & 73.39 & 20.75 & 47.07 \\
& GRPO + \rtwo~ & 59.29 & 40.80 & 50.04 & \textbf{73.92} & 20.90 & 47.41 \\
\rowcolor{TableHighlight}
& \textbf{GRPO + \method} & \textbf{63.76} & \textbf{43.20} & \textbf{53.48} & 72.71 & \textbf{22.50} & \textbf{47.60} \\
\bottomrule
\end{tabular}
}
\vspace{8pt}
\caption{\textbf{Downstream reasoning accuracy on Moonlight-16B-A3B and OLMoE-1B-7B.} Columns are grouped by model. \rtwo~refers to routing replay for simplicity. Bold values mark the best accuracy in each metric column within an off-policy strength, and shaded rows mark \method.}
\label{tab:cross_model_pr2}
\vspace{-16pt}
\end{table}

\vspace{-2mm}
\paragraph{Cross-Model Evaluation.}
Table~\ref{tab:cross_model_pr2} extends the comparison to Moonlight-16B-A3B and OLMoE-1B-7B under matched off-policy strengths. On Moonlight-16B-A3B, GRPO + \method~obtains the best average accuracy in all three settings, improving over routing replay by $1.19\%$, $1.19\%$, and $3.44\%$ points under off-2, off-4, and off-8, respectively. The largest gain appears under the strongest rollout reuse, again matching the router-staleness view that frozen replay routes drift farther from the current router as $\kappa$ increases.
On OLMoE-1B-7B, the absolute gains are smaller, but GRPO + \method~still achieves the best average accuracy across off-2, off-4, and off-8. These results suggest that predictive routing replay is not specific to a single MoE backbone or dataset.

\vspace{-2mm}
\subsection{Training Stability}
\paragraph{Training Dynamics.}
Figure~\ref{fig:main_training_dynamics} compares off-2 training dynamics on Qwen3-30B-A3B-Base for GRPO, routing replay, and \method. We use these curves to examine whether repeated rollout reuse only affects final accuracy or also changes the optimization trajectory. Vanilla GRPO exhibits policy-gradient loss spikes and abrupt reward drops under rollout reuse. Routing replay reduces the most severe spikes, but still shows late-stage oscillations in clipping rates and entropy. In contrast, \method~keeps clipping rates substantially lower and yields smoother reward growth.
The response-entropy and length curves show a similar pattern. \method~avoids abrupt entropy collapse while allowing response length to increase steadily, suggesting that the policy can improve reasoning behavior without entering unstable high-clipping updates. This supports the interpretation that predictive routing replay stabilizes the optimization process.

\begin{figure*}[ht]
    \vspace{-2mm}
    \centering
    \includegraphics[width=0.96\linewidth]{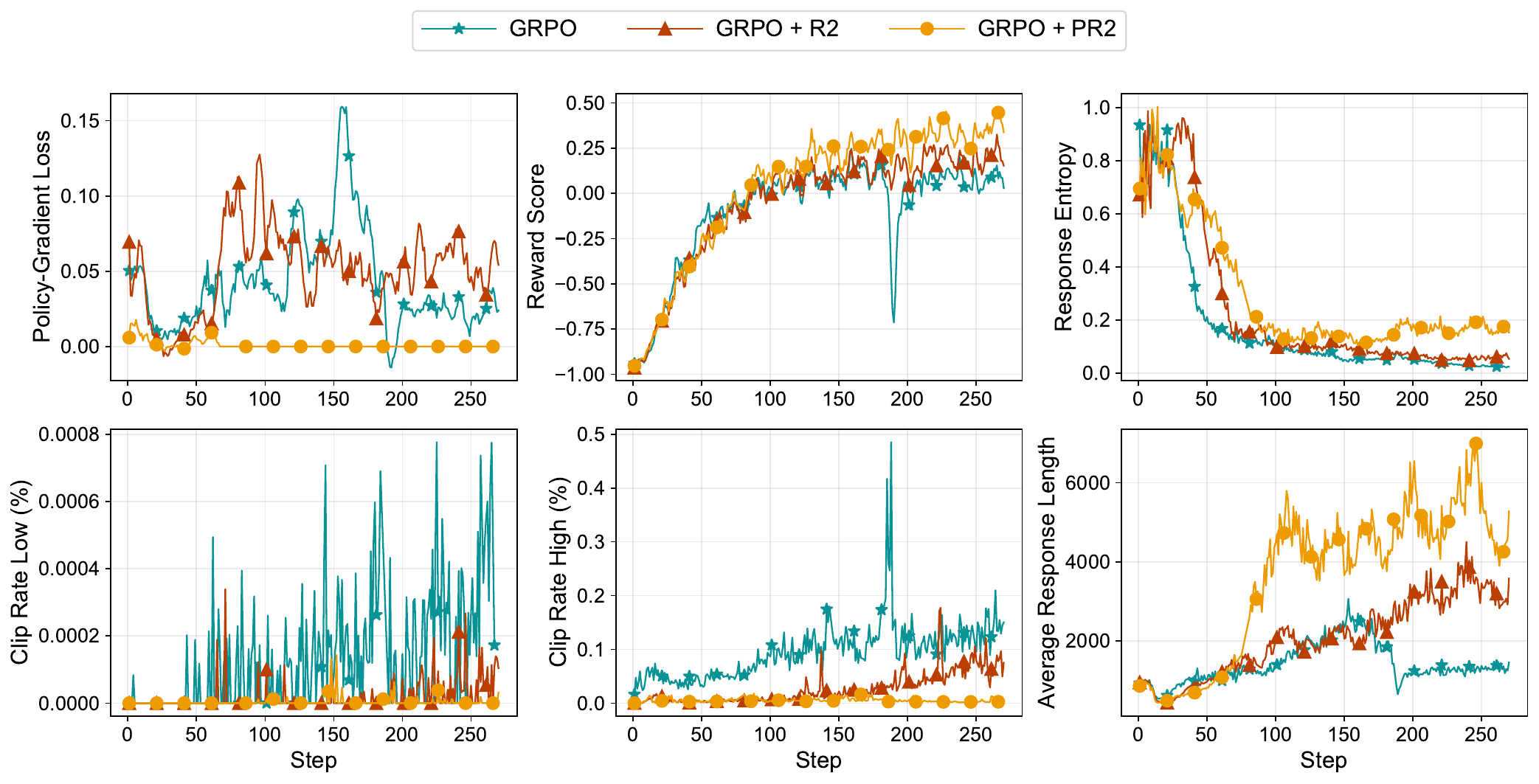}
    \vspace{-2mm}
    \caption{\textbf{Training dynamics on Qwen3-30B-A3B-Base under off-2.} We compare GRPO, routing replay, and \method~on policy-gradient loss, reward, response entropy, clipping rates, and response length. Smoother curves and lower clipping volatility indicate more stable updates.}
    \vspace{-2mm}
    \label{fig:main_training_dynamics}
\end{figure*}

\paragraph{Optimization Behavior.}
The curves in Figure~\ref{fig:main_training_dynamics} illustrate how routing mismatch affects PPO-style optimization. In off-policy MoE RL, stale routing can produce extreme token-level importance ratios, which appear as elevated clipping rates and can further lead to loss spikes and reward regressions. Routing replay mitigates the most severe instability by enforcing replay consistency, but its cached route can still become stale as repeated updates move the router. \method~improves the training trajectory itself by replaying predicted routes that better match short-horizon router evolution.

\subsection{Routing Prediction Analysis}
\label{sec:predictive_routing_analysis}

\paragraph{Top-$k$ Agreement and Route KL.}
Figure~\ref{fig:ana_topk_acc} reports top-$k$ agreement and route KL on Qwen3-30B-A3B-Base. These two measurements are directly tied to the behavior of the evolution predictor. We compute agreement as $|\hat{\mathcal{I}}_t^{(l)}\cap \mathcal{I}_t^{(l)}|\ \big/\ {k}$ and average over tokens and layers. Vanilla GRPO shows a clear drop in top-$k$ agreement and a rapid increase in route KL, indicating severe router drift during training. Routing replay reduces this drift but still becomes stale in the later stage. In contrast, \method~maintains higher agreement and lower KL across off-2, off-4, and off-8, suggesting that the predicted replay routes better track short-horizon router evolution under repeated rollout reuse.

\begin{figure}[!htbp]
    \centering
    \begin{subfigure}[t]{0.60\linewidth}
        \centering
        \includegraphics[width=\linewidth]{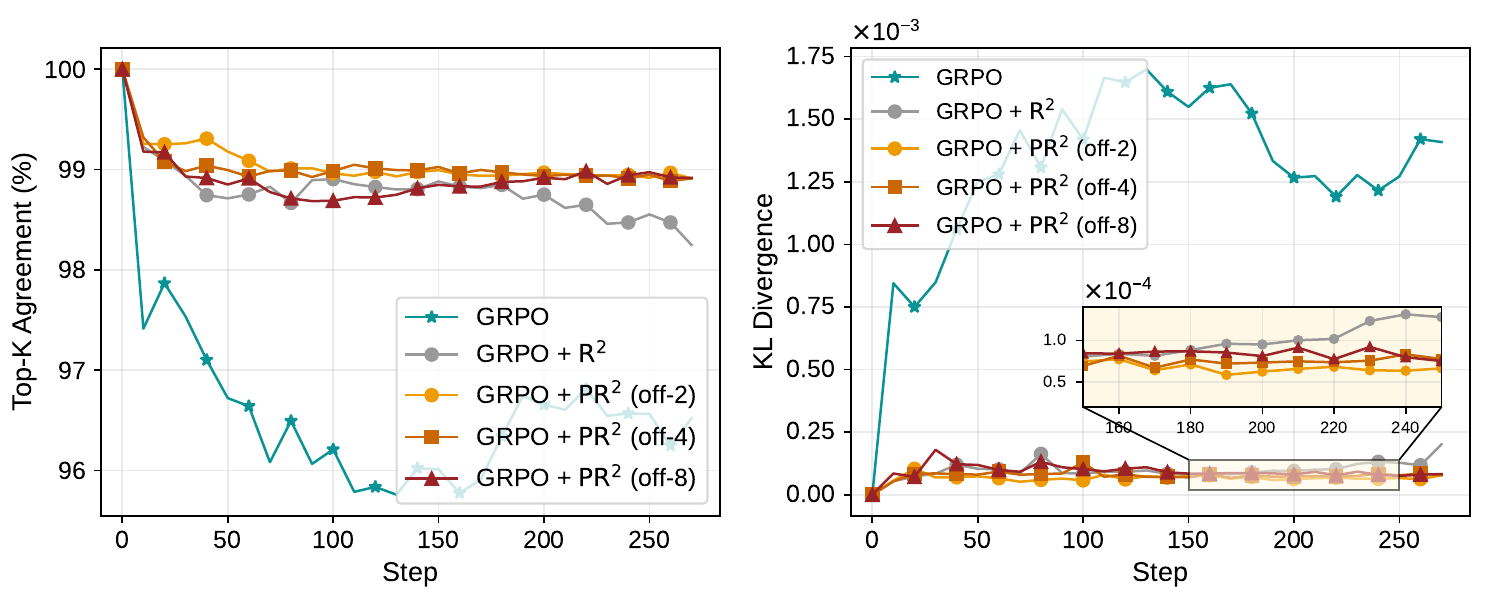}
        \subcaption{Top-$k$ agreement and KL divergence curves.}
        \label{fig:ana_topk_acc}
    \end{subfigure}
    \hfill
    \begin{subfigure}[t]{0.37\linewidth}
        \centering
        \includegraphics[width=\linewidth]{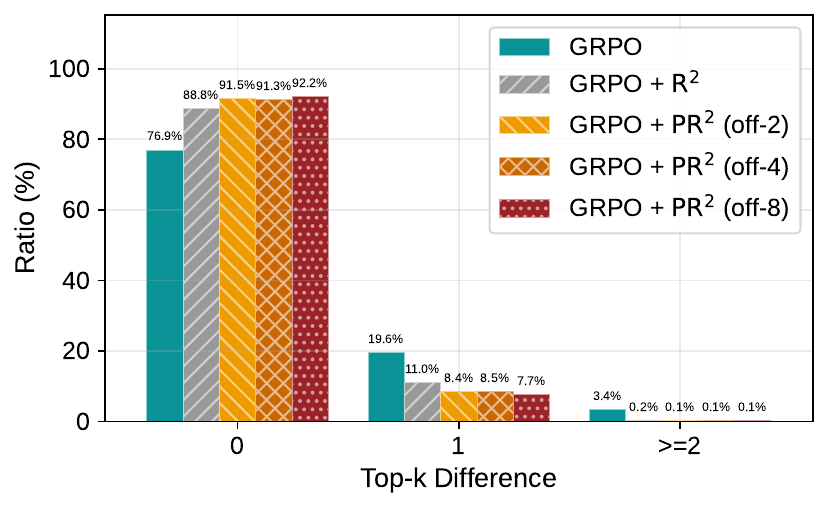}
        \subcaption{Routing deviation distribution.}
        \label{fig:ana_topk_dev}
    \end{subfigure}
    \vspace{8pt}
    \caption{\textbf{Routing prediction behavior during training.} \textbf{(a)} Top-$k$ agreement between cached and current routing indices, together with route KL divergence curves. Higher agreement and lower KL indicate better staleness control by \method. \textbf{(b)} Routing deviation distribution at the final step. More mass at $0$ difference and less mass at $\geq1$ differences indicate better route tracking. Baselines are from off-2 runs unless otherwise specified.}
    \vspace{-1em}
    \label{fig:ana_routing_tracking}
\end{figure}

\paragraph{Deviation Tails.}
Average agreement can hide rare but severe route flips. Figure~\ref{fig:ana_topk_dev} therefore reports the deviation count $k-|\hat{\mathcal{I}}_t^{(l)}\cap \mathcal{I}_t^{(l)}|$ on Qwen3-30B-A3B-Base. Compared with GRPO, \method~shifts substantially more mass to zero deviation. The zero-deviation ratio increases from $76.9\%$ for GRPO and $88.8\%$ for routing replay to $91.5\%$, $91.3\%$, and $92.2\%$ under off-2, off-4, and off-8, respectively. Correspondingly, the $1$-slot mismatch ratio decreases from $19.6\%$ and $11.0\%$ to $8.4\%$, $8.5\%$, and $7.7\%$, while the fraction of cases with $\geq2$ mismatched slots is reduced to only $0.1\%$. These results show that \method~not only improves average route tracking, but also suppresses large expert mismatches that can destabilize fixed-route updates.

\section{Conclusion}
\label{sec:conclusion}
This work identifies router drift as a source of instability in off-policy RL for MoE-based LLMs and introduces \textbf{Predictive Routing Replay (\method)}, which predicts short-horizon expert indices while preserving deterministic replay. \method addresses a central tension in routing replay: fixed routes help stabilize PPO-style importance estimation, but stale routes can prevent training from following router evolution after repeated policy updates. By training a lightweight router-side predictor on cached route-recording features, \method reduces router staleness while keeping the replay route fixed during training. Our theoretical analysis motivates the predictive KL objective through a staleness-controlled gradient-deviation bound, and our experiments show that \method improves reasoning accuracy, stabilizes PPO-style optimization, and yields closer route tracking across multiple MoE backbones and off-policy strengths. These results suggest that anticipating short-horizon router evolution provides a practical and effective replay alternative for stable RL training of MoE-based LLMs.

\section*{Acknowledgments}
The authors gratefully acknowledge support from the AMD University Program. We also thank the developers and maintainers of the open-source reinforcement-learning and MoE infrastructure that supported this work.

\bibliographystyle{plainnat}
\bibliography{references}

\clearpage
\appendix
\onecolumn
\raggedbottom

\section{RL Objectives and Notation}
\label{app:ppo}

\subsection{Notation and Setup}
We consider an autoregressive policy $\pi_\theta$ generating completion $y=(y_1,\dots,y_T)$ for prompt $x$.
At token $t$, the policy conditions on the prefix $(x,y_{<t})$ and predicts the next token $y_t$.
We use $\pi_{\theta_{\text{old}}}$ for the old policy snapshot that produced the rollout trajectories and $\pi_\theta$ for the current training policy. 
We use
$\mathrm{Clip}(u,1-\epsilon,1+\epsilon)$
for PPO clipping~\citep{schulman2017proximal,yu2025dapo}.

\subsection{Proximal Policy Optimization (PPO)}
PPO optimizes a clipped surrogate objective with a token-level ratio
\begin{equation}
r_t(\theta) = \frac{\pi_\theta(y_t\mid x,y_{<t})}{\pi_{\theta_{\text{old}}}(y_t\mid x,y_{<t})}.
\label{eq:ppo_ratio}
\end{equation}
Given advantage $A_t$,
\begin{equation}
\mathcal{L}_{\text{PPO}}(\theta) = \mathbb{E}\!\left[\sum_{t=1}^{T} \min\!\Big( r_t(\theta)\,A_t,\ \mathrm{Clip}(r_t(\theta),1-\epsilon,1+\epsilon)\,A_t \Big)\right].
\label{eq:ppo_obj}
\end{equation}
A KL regularizer to reference policy $\pi_{\text{ref}}$ is often added.
\begin{equation}
\mathcal{L}_{\text{PPO+KL}}(\theta) = \mathcal{L}_{\text{PPO}}(\theta) - \beta\,\mathbb{E}\!\left[\sum_{t=1}^{T} D_{\mathrm{KL}}\!\big(\pi_\theta(\cdot\mid x,y_{<t})\,\big\|\,\pi_{\text{ref}}(\cdot\mid x,y_{<t})\big)\right].
\label{eq:ppo_kl}
\end{equation}

\subsection{Group Relative Policy Optimization (GRPO)}
GRPO uses group-relative advantages without value-function training~\citep{shao2024deepseekmath,guo2025deepseek}.
Given a group of $G$ rollouts $\{y^{(g)}\}_{g=1}^G$ for prompt $x$, GRPO defines
\begin{equation}
A_t^{(g)} = \frac{R(x,y^{(g)}) - \frac{1}{G}\sum_{g'=1}^G R(x,y^{(g')})}{\mathrm{Std}\big(\{R(x,y^{(g')})\}_{g'=1}^G\big) + \epsilon_A}.
\end{equation}
The GRPO objective follows the PPO clipping form with these group-relative advantages.
\begin{equation}
\mathcal{L}_{\text{GRPO}}(\theta) = \mathbb{E}\!\left[\frac{1}{G}\sum_{g=1}^G\sum_{t=1}^T \min\!\Big( r_t^{(g)}(\theta)\,A_t^{(g)},\ \mathrm{Clip}(r_t^{(g)}(\theta),1-\epsilon,1+\epsilon)\,A_t^{(g)} \Big)\right],
\end{equation}
where $r_t^{(g)}(\theta)=\pi_\theta(y_t^{(g)}\mid x,y_{<t}^{(g)})\big/\pi_{\theta_{\text{old}}}(y_t^{(g)}\mid x,y_{<t}^{(g)})$ is the token-level importance ratio for the $g$-th group completion, and the leading $1/G$ keeps the loss magnitude scale-invariant to group size.

In our implementation we follow the Clip-Higher variant of DAPO~\citep{yu2025dapo}, which replaces the symmetric $\epsilon$ in the clip with an asymmetric pair $(\epsilon_{\text{low}}, \epsilon_{\text{high}})$:
\begin{equation}
\mathcal{L}_{\text{GRPO}}^{\text{DAPO}}(\theta) = \mathbb{E}\!\left[\frac{1}{G}\sum_{g=1}^G\sum_{t=1}^T \min\!\Big( r_t^{(g)}(\theta)\,A_t^{(g)},\ \mathrm{Clip}(r_t^{(g)}(\theta),1-\epsilon_{\text{low}},1+\epsilon_{\text{high}})\,A_t^{(g)} \Big)\right].
\label{eq:grpo_dapo}
\end{equation}

\section{Theoretical Details}
\label{app:theory}

\subsection{Routing Replay Gradient Deviation Bound}
\label{app:bias}

We restate the TV-based gradient-deviation bound used in Section~\ref{sec:method}.
Let $\psi_t(\mathcal R_t,\theta)$ denote the token-level PPO gradient contribution evaluated under fixed route $\mathcal R_t$. For any route distribution $P_t$, define
\begin{equation}
g_t(P_t,\theta) := \mathbb{E}_{\mathcal R_t\sim P_t(\cdot\mid x,y_{<t})}\!\left[\psi_t(\mathcal R_t,\theta)\right].
\end{equation}
For any two route distributions $P_t$ and $Q_t$, assume the fixed-route gradient kernel is locally bounded on the parameter region visited during training, with $M_t\ge\|\psi_t(\mathcal R_t,\theta)\|$ for all $\mathcal R_t$. The variational characterization of total variation gives
\begin{equation}
\|g_t(P_t,\theta)-g_t(Q_t,\theta)\| \le 2M_t\,D_{\mathrm{TV}}\!\left(P_t,Q_t\right).
\label{eq:tv_grad_bias}
\end{equation}
Here and below, common conditioning on $(x,y_{<t})$ is suppressed inside divergences. Eq.~\eqref{eq:tv_grad_bias} links router staleness and route-induced gradient deviation without depending on the specific PPO clipping form.

\subsection{\texorpdfstring{Soft \method~Surrogate and Top-$k$ Bridge}{Soft \method~Surrogate and Top-k Bridge}}
\label{app:prop-proof}

We analyze the \method~predictive loss through a layer-wise relaxation. Instead of treating a layer route as a top-$k$ set, the relaxed route draws one categorical expert variable $Z_t^{(l)}\in\{1,\ldots,N\}$ from $\rho_t^{(l)}$ or $\hat\rho_t^{(l)}$. Let $Z_t=(Z_t^{(1)},\ldots,Z_t^{(L)})$.
\begin{equation}
P_t^{\rho}(Z_t) = \prod_{l=1}^{L}\rho_t^{(l)}(Z_t^{(l)}), \qquad P_t^{\hat\rho}(Z_t) = \prod_{l=1}^{L}\hat\rho_t^{(l)}(Z_t^{(l)}).
\label{eq:categorical_route_relaxation}
\end{equation}
This relaxation is not the hard top-$k$ replay distribution. It is the soft distributional object optimized by Eq.~\eqref{eq:pr2_loss}, where the current routing distribution is evaluated under the same replay-conditioned forward pass used for training. In this way, the bound characterizes staleness within the fixed-route training computation rather than the fully free-running MoE computation. The analysis uses the layer-summed definition in Eq.~\eqref{eq:pr2_loss}. If one reports a layer-averaged variant, the corresponding bound carries an additional factor $L$ inside the square root.

\begin{proposition}[Predictive loss controls soft route-gradient deviation]
\label{prop:pr2_bias}
Under the local boundedness assumption $\|\psi_t(Z_t,\theta)\|\le M$ for all tokens $t$ and relaxed routes $Z_t$, and under the categorical-route relaxation in Eq.~\eqref{eq:categorical_route_relaxation},
\begin{equation}
\mathbb E_t\!\left[\left\|g_t(P_t^{\rho},\theta)-g_t(P_t^{\hat\rho},\theta)\right\|\right] \le M\sqrt{2\,\mathcal L_{\methodmath}}.
\end{equation}
\end{proposition}

\begin{proof}
Let $P_t^{\rho}$ and $P_t^{\hat\rho}$ be the categorical-route distributions in Eq.~\eqref{eq:categorical_route_relaxation}, and write $\Delta g_t=g_t(P_t^{\rho},\theta)-g_t(P_t^{\hat\rho},\theta)$. Applying Eq.~\eqref{eq:tv_grad_bias} gives $\|\Delta g_t\| \le 2M_t\,D_{\mathrm{TV}}(P_t^{\rho},P_t^{\hat\rho})$, and Pinsker's inequality yields $D_{\mathrm{TV}}(P_t^{\rho},P_t^{\hat\rho}) \le \sqrt{\tfrac{1}{2}D_{\mathrm{KL}}(P_t^{\rho}\|P_t^{\hat\rho})}$, so
\begin{equation}
\|\Delta g_t\| \le 2M_t\sqrt{\tfrac{1}{2}\,D_{\mathrm{KL}}\!\left(P_t^{\rho}\|P_t^{\hat\rho}\right)}.
\end{equation}
Under the layer-wise categorical relaxation, $D_{\mathrm{KL}}(P_t^{\rho}\|P_t^{\hat\rho}) = \sum_{l=1}^{L} D_{\mathrm{KL}}(\rho_t^{(l)}\|\hat\rho_t^{(l)})$, hence
\begin{equation}
\|\Delta g_t\| \le 2M_t\sqrt{\tfrac{1}{2}\sum_{l=1}^{L}D_{\mathrm{KL}}\!\left(\rho_t^{(l)}\|\hat{\rho}_t^{(l)}\right)}.
\end{equation}
Taking expectation over $t$, using $M_t\le M$, and applying Jensen's inequality to the $\sqrt{\cdot}$, we obtain $\mathbb E_t[\|\Delta g_t\|] \le M\sqrt{2\sum_{l=1}^{L}\mathbb E_t[D_{\mathrm{KL}}(\rho_t^{(l)}\|\hat\rho_t^{(l)})]} = M\sqrt{2\,\mathcal L_{\methodmath}}$, where the last equality uses Eq.~\eqref{eq:pr2_loss}.
\end{proof}

\paragraph{Remark on $k$-Slot Relaxations.}
The categorical relaxation above samples one expert variable per layer because it is the soft object directly matched by Eq.~\eqref{eq:pr2_loss}. A closer $k$-slot relaxation could sample $(Z_{t,1}^{(l)},\ldots,Z_{t,k}^{(l)})$ from $\rho_t^{(l)}$. The same TV and Pinsker argument would introduce the corresponding slot factor, while Lemma~\ref{lemma:topk_regret} gives a deterministic support-level (activated expert indices) guarantee for the hard top-$k$ expert support used by \method.

\begin{lemma}[Predictive KL controls top-$k$ support regret]
\label{lemma:topk_regret}
Consider a single layer and write $p=\rho_t^{(l)}$ and $\hat p=\hat\rho_t^{(l)}$. Let $S^\star=\mathrm{TopK}(p,k)$ and $\hat S=\mathrm{TopK}(\hat p,k)$, with $p(S)=\sum_{j\in S}p_j$. Then
\[
p(S^\star)-p(\hat S)
\le
\sqrt{2\,D_{\mathrm{KL}}(p\|\hat p)}.
\]
Consequently, if the current router has top-$k$ margin $\Delta_k(p)=p_{(k)}-p_{(k+1)}>0$ and $D_{\mathrm{KL}}(p\|\hat p)<\Delta_k(p)^2/2$, then $\mathrm{TopK}(\hat p,k)=\mathrm{TopK}(p,k)$.
\end{lemma}

\begin{proof}[Proof of Lemma~\ref{lemma:topk_regret}]
Let $S^\star=\mathrm{TopK}(p,k)$ and $\hat S=\mathrm{TopK}(\hat p,k)$. Since $\hat S$ maximizes $\hat p(S)$ over all $k$-element supports, we have $\hat p(\hat S)\ge \hat p(S^\star)$. Decomposing $p(S^\star)-p(\hat S)$ into three telescoping terms,
\begin{align*}
p(S^\star)-p(\hat S)
&= \big(p(S^\star)-\hat p(S^\star)\big) + \big(\hat p(S^\star)-\hat p(\hat S)\big) + \big(\hat p(\hat S)-p(\hat S)\big) \\
&\le D_{\mathrm{TV}}(p,\hat p) + 0 + D_{\mathrm{TV}}(p,\hat p) \\
&\le \sqrt{2\,D_{\mathrm{KL}}(p\|\hat p)},
\end{align*}
where the second line uses $\hat p(S^\star)-\hat p(\hat S)\le 0$ together with the variational bound $|p(S)-\hat p(S)|\le D_{\mathrm{TV}}(p,\hat p)$ for any subset $S$, and the third line uses Pinsker's inequality.

For the margin statement, the variational definition of total variation gives the elementwise bound
\begin{equation*}
\|p-\hat p\|_\infty \le D_{\mathrm{TV}}(p,\hat p),
\end{equation*}
obtained by taking the singleton $A=\{j\}$ in $D_{\mathrm{TV}}(p,\hat p)=\max_A|p(A)-\hat p(A)|$, and Pinsker's inequality further yields
\begin{equation*}
D_{\mathrm{TV}}(p,\hat p) \le \sqrt{\tfrac{1}{2}\,D_{\mathrm{KL}}(p\|\hat p)}.
\end{equation*}
Hence whenever $\Delta_k(p)=p_{(k)}-p_{(k+1)}>0$ and $D_{\mathrm{KL}}(p\|\hat p)<\Delta_k(p)^2/2$, we obtain $\|p-\hat p\|_\infty<\Delta_k(p)/2$. Every top-$k$ expert under $p$ therefore remains above every non-top-$k$ expert under $p$ when scored by $\hat p$, so $\mathrm{TopK}(\hat p,k)=\mathrm{TopK}(p,k)$.
\end{proof}

\section{Rollout-Engine Behavior Policy}
\label{app:r3_pr3}

The main text in Section~\ref{sec:preliminary} treats the trajectory source as a single old policy snapshot $\pi_{\theta_{\text{old}}}$. In disaggregated training pipelines, rollouts are actually produced by a separate rollout engine whose effective behavior policy $\mu_{\theta_{\text{old}}}$ can differ from the training-engine snapshot $\pi_{\theta_{\text{old}}}$ at the implementation level. Following~\cite{zheng2025stabilizing}, the token-level importance ratio admits the decomposition
\begin{equation}
    r_t(\theta)
    =
    \underbrace{
    \frac{\pi_{\theta_{\text{old}}}(y_t\mid x,y_{<t})}{\mu_{\theta_{\text{old}}}(y_t\mid x,y_{<t})}
    }_{\text{system discrepancy}}
    \cdot
    \underbrace{
    \frac{\pi_{\theta}(y_t\mid x,y_{<t})}{\pi_{\theta_{\text{old}}}(y_t\mid x,y_{<t})}
    }_{\text{policy staleness}},
    \label{eq:ratio_decomp_app}
\end{equation}
which factorizes the rollout/training-engine gap from update-induced staleness. Each RL step then involves three MoE forward passes: a rollout pass under $\mu_{\theta_{\text{old}}}$ in the rollout engine that generates trajectories, a log-prob pass under $\pi_{\theta_{\text{old}}}$ in the training engine that evaluates the snapshot likelihoods used as the importance-ratio denominator, and a gradient pass under $\pi_\theta$ in the training engine that evaluates and backpropagates the surrogate loss. Each pass can pick different expert indices, and replay schemes resolve this by reusing cached indices across the latter two training-engine passes.

\subsection{Rollout Routing Replay}
Rollout routing replay~\citep{ma2025stabilizing} caches the rollout-engine route $\mathcal{I}_{\text{old},t}^{(l)}$ produced by $\mu_{\theta_{\text{old}}}$ at each token $t$ and layer $l$ during the rollout pass, and replays it in \emph{both} training-engine passes. The log-prob pass under $\pi_{\theta_{\text{old}}}$ and the gradient pass under $\pi_\theta$ each substitute $\mathcal{I}_{\text{old},t}^{(l)}$ for their own top-$k$ selection, so the two forward passes share the same expert indices. Reusing the indices is what makes the snapshot ratio $\pi_\theta(y_t\mid x,y_{<t})/\pi_{\theta_{\text{old}}}(y_t\mid x,y_{<t})$ in the importance weight well-defined despite the rollout-training engine gap, at the cost of freezing the indices to $\mathcal{I}_{\text{old},t}^{(l)}$ throughout.

\subsection{\methodthree: Predictive Rollout Routing Replay}
\methodthree~inserts the \method~prediction mechanism before rollout routing replay's frozen indices. The route-recording phase runs under $\mu_{\theta_{\text{old}}}$ in the rollout engine. We reuse the \method~conventions of Section~\ref{sec:method}, with $h_{\text{old},t}^{(l)}$, $p_{\text{old},t}^{(l)} = h_{\text{old},t}^{(l)} W_{\text{old}}^{(l)}$, $E_{\text{old},j}^{(l)}$ now denoting rollout-engine quantities. The predictive distribution and predicted top-$k$ indices follow Eqs.~\eqref{eq:pr2_predict}--\eqref{eq:pr2_route},
\begin{equation}
\hat\rho_t^{(l)} = \mathrm{Softmax}\!\left( p_{\text{old},t}^{(l)} + h_{\text{old},t}^{(l)} W_p^{(l)} \right), \qquad \hat{\mathcal{I}}_t^{(l)} = \mathrm{TopK}\!\left(\hat\rho_t^{(l)},\, k\right),
\end{equation}
and the route-recording output reuses the same form as Eq.~\eqref{eq:pr2_old_output},
\begin{equation}
o_{\text{old},t}^{(l)} = \sum_{j \in \hat{\mathcal{I}}_t^{(l)}} \hat\rho_{t,j}^{(l)}\, E_{\text{old},j}^{(l)}\!\left(h_{\text{old},t}^{(l)}\right),
\end{equation}
which drives autoregressive generation. We cache $\hat{\mathcal{I}}_t^{(l)}$ together with the route-recording features $(h_{\text{old},t}^{(l)},\, p_{\text{old},t}^{(l)})$. Both training-engine passes then replay $\hat{\mathcal{I}}_t^{(l)}$ as the MoE expert indices. The log-prob pass under $\pi_{\theta_{\text{old}}}$ uses the cached indices with the snapshot's own routing weights and expert parameters, and the gradient pass under $\pi_\theta$ follows Eq.~\eqref{eq:pr2_output} with the cached indices. The predictor is bypassed in both replay passes and is supervised by the KL objective in Eq.~\eqref{eq:pr2_loss}, evaluated on the cached rollout-engine features against the routing distribution observed under $\pi_\theta$ on the same batch. \methodthree~thus stands to rollout routing replay~exactly as \method~stands to routing replay. The predicted indices replace the frozen old-snapshot indices, and are replayed across training-engine forward passes. The \methodthree~versus rollout routing replay~comparison on Qwen3-30B-A3B-Base is reported in Table~\ref{tab:qwen_r3_comparison}.

\section{Predictive Routing Replay Pseudocode}
\label{app:pr2_algo}

Algorithms~\ref{alg:pr2_inference} and~\ref{alg:pr2_training} give the per-token, per-layer route-recording and replay steps used in Section~\ref{sec:method}. The per-layer predictive losses are aggregated as in Eq.~\eqref{eq:pr2_loss}.

\begin{algorithm}[H]
  \caption{\method~route recording under the rollout policy $\pi_{\theta_{\text{old}}}$.}
  \label{alg:pr2_inference}
  \begin{algorithmic}[1]
    \REQUIRE
      Token position $\tau$, layer index $l$, hidden feature $h_{\text{old}}$, router weights $W_{\text{old}}^{(l)}$, predictor $W_p^{(l)}$, experts $\{E_{\text{old},j}^{(l)}\}_{j=1}^{N}$, index cache $\mathcal{P}$, and feature cache $\mathcal{F}$.
    \ENSURE
      MoE layer output $o_{\text{old}}$.
    \STATE Form old-snapshot logits and predictor bias $p_{\text{old}}\leftarrow h_{\text{old}}W_{\text{old}}^{(l)}$, $b\leftarrow h_{\text{old}}W_p^{(l)}$.
    \STATE Construct the predictive distribution and predicted top-$k$ indices $\hat{\rho}\leftarrow\mathrm{Softmax}(p_{\text{old}}+b)$, $\hat{\mathcal{I}}\leftarrow\mathrm{TopK}(\hat{\rho},k)$.
    \STATE Cache the predicted indices $\mathcal{P}[(\tau,l)]\leftarrow\hat{\mathcal{I}}$, and store route-recording features $\mathcal{F}[(\tau,l)]\leftarrow(h_{\text{old}},p_{\text{old}})$ when $(\tau,l)$ is retained by the fixed-length feature cache.
    \STATE Compute the layer output on the predicted indices $o_{\text{old}}\leftarrow\sum_{j\in\hat{\mathcal{I}}}\hat{\rho}_{j}\,E_{\text{old},j}^{(l)}(h_{\text{old}})$.
  \end{algorithmic}
\end{algorithm}

\begin{algorithm}[H]
  \caption{\method~replay and predictor update under the training policy $\pi_{\theta}$.}
  \label{alg:pr2_training}
  \begin{algorithmic}[1]
    \REQUIRE
      Mini-step index $i\in\{1,\dots,\kappa\}$, token position $\tau$, layer index $l$, hidden feature $h$, router weights $W^{(l)}$, predictor $W_p^{(l)}$, experts $\{E_j^{(l)}\}_{j=1}^{N}$, index cache $\mathcal{P}$, feature cache $\mathcal{F}$.
    \ENSURE
      MoE layer output $o$, per-layer predictive loss $\ell_{\methodmath}^{(l)}$.
    \STATE Compute current routing weights and load the cached indices $p\leftarrow hW^{(l)}$, $\rho\leftarrow\mathrm{Softmax}(p)$, $\hat{\mathcal{I}}\leftarrow\mathcal{P}[(\tau,l)]$.
    \STATE Compute the layer output with current experts on the cached indices $o\leftarrow\sum_{j\in\hat{\mathcal{I}}}\rho_j\,E_j^{(l)}(h)$.
    \IF{$i=1 \lor (\tau,l)\notin\mathcal{F}$}
      \STATE Skip the predictive loss on the on-policy first mini-step or when no cached feature is available $\ell_{\methodmath}^{(l)}\leftarrow0$.
    \ELSE
      \STATE Reconstruct the predictive distribution from cached features $(h_{\text{old}},p_{\text{old}})\leftarrow\mathcal{F}[(\tau,l)]$, $\hat{\rho}\leftarrow\mathrm{Softmax}(p_{\text{old}}+h_{\text{old}}W_p^{(l)})$.
      \STATE Compute the per-layer predictive KL loss $\ell_{\methodmath}^{(l)}\leftarrow D_\mathrm{KL}\Big(\langle\rho\rangle\,\big\|\,\hat{\rho}\Big)$.
    \ENDIF
  \end{algorithmic}
\end{algorithm}

\FloatBarrier
\section{Implementation Details}
\label{app:train_detail}

\subsection{Model and Hyperparameters}

All three runs share a common off-policy schedule: per model we fix a rollout batch size $\mathcal{B}_{\text{global}}$ and let the per-update batch size $\mathcal{B}_{\text{update}}$ take three values to realize off-2, off-4, and off-8. Unless noted otherwise, the asymmetric clip ratios are $\epsilon_{\text{low}}=0.2$ and $\epsilon_{\text{high}}=0.28$, the SGLang oversampling ratio is $0.1$, and predictors are zero-initialized. Following the trade-off between the training update mini-steps and learning rate described by~\cite{ma2025stabilizing}, we use a smaller per-update learning rate under stronger off-policy reuse. Each rollout batch supports $\kappa$ inner mini-step updates under off-$\kappa$, so a smaller per-update learning rate keeps the cumulative parameter drift over a rollout comparable across off-policy strengths.

\paragraph{Qwen3-30B-A3B-Base.}
Rollout uses $\mathcal{B}_{\text{global}}=64$ with 8 responses per prompt and a maximum generation length of $16\mathrm{K}$ tokens. We sweep $\mathcal{B}_{\text{update}}\in\{32,16,8\}$ for off-2, off-4, and off-8 with learning rates $\{2,\,1.5,\,1\}\times10^{-6}$ and predictor learning-rate multipliers $\{10^{4},\,10^{3},\,10^{3}\}$. Each run uses 32 NVIDIA H200 GPUs for about 72 hours.

\paragraph{Moonlight-16B-A3B.}
Rollout uses $\mathcal{B}_{\text{global}}=32$ with 16 responses per prompt and a maximum generation length of $1\mathrm{K}$ tokens. We sweep $\mathcal{B}_{\text{update}}\in\{16,8,4\}$ for off-2, off-4, and off-8 with learning rates $\{5,\,3.7,\,2.5\}\times10^{-7}$ and predictor learning-rate multipliers $\{10^{2},\,5\times10^{1},\,5\times10^{1}\}$. Each run uses 32 NVIDIA H200 GPUs for about 24 hours.

\paragraph{OLMoE-1B-7B.}
Rollout uses $\mathcal{B}_{\text{global}}=32$ with 8 responses per prompt and a maximum generation length of $1\mathrm{K}$ tokens. We sweep $\mathcal{B}_{\text{update}}\in\{16,8,4\}$ for off-2, off-4, and off-8 with learning rates $\{5,\,3.7,\,2.5\}\times10^{-7}$ and a fixed predictor learning-rate multiplier $10^{3}$. The clip ratios are $\epsilon_{\text{low}}=0.2$ and $\epsilon_{\text{high}}=0.1$. Each run uses 8 NVIDIA RTX PRO 6000 GPUs for about 4 hours.

\paragraph{Baseline Settings.}
GRPO follows DAPO Clip-Higher without KL regularization. GSPO uses $\epsilon_{\text{low}}=3\times10^{-4}$, $\epsilon_{\text{high}}=4\times10^{-4}$, and KL coefficient $10^{-3}$. Routing replay~records routes from the old policy snapshot. Rollout routing replay~shares \method's data, optimizer, off-policy schedule, and validations but replays routes produced by the rollout-engine behavior policy.

\paragraph{\method~Implementation.}
\method~adds an evolution predictor to each router without modifying the PPO objective. Runtime overhead is limited to a single router-side projection during route recording and the predictive-loss evaluation on a bounded feature cache during training. The full predicted top-$k$ indices are always cached for replay. Route-recording features $(h_{\text{old}},p_{\text{old}})$ used to train the predictor are cached only for a fixed number of tokens, with cache length $T_c=2\mathrm{K}$ for Qwen3-30B-A3B-Base (whose maximum generation length is $T=16\mathrm{K}$) and $T_c=T=1\mathrm{K}$ for Moonlight-16B-A3B and OLMoE-1B-7B. For Qwen3-30B-A3B-Base, the $T_c$ retained positions are drawn by uniform sub-sampling along the rollout, so the predictor sees uniform coverage of the full $16\mathrm{K}$ context rather than only the fixed positions of tokens.

\subsection{Resource Details}
\label{app:overhead_hparams}

\begin{table}[!htbp]
\centering
\scriptsize
\setlength{\tabcolsep}{6pt}
\renewcommand{\arraystretch}{1.08}
\resizebox{0.98\linewidth}{!}{%
\begin{tabular}{l c c c c c}
\toprule
\rowcolor{TableHeader}
\textbf{Model}
& \textbf{$(L,d,N,k,s,m,T,T_c)$}
& \textbf{Index Cache}
& \textbf{Feature Cache}
& \textbf{FLOPs}
& \textbf{FFN Ratio} \\
\midrule
Qwen3
& $(48,2048,128,8,0,768,16\mathrm{K},2\mathrm{K})$
& 25.2 MB
& 453 MB
& 25.2M
& 0.69\% \\
Moonlight
& $(26,2048,64,6,2,1408,1\mathrm{K},1\mathrm{K})$
& 0.64 MB
& 116 MB
& 6.82M
& 0.19\% \\
OLMoE
& $(16,2048,64,8,0,1024,1\mathrm{K},1\mathrm{K})$
& 0.52 MB
& 71.3 MB
& 4.19M
& 0.26\% \\
\bottomrule
\end{tabular}%
}
\vspace{8pt}
\caption{\method~cache and route-recording forward-compute overhead. Index cache uses $4TLk$ bytes for int32 expert indices, feature cache uses $T_cL(2d+4N)$ bytes for BF16 hidden features and FP32 router logits, and forward compute reports the extra $2LdN$ FLOPs per token from the predictor projection. Cache values are per response, and FLOPs values are per token.}
\label{tab:model_config}
\end{table}

\paragraph{Predictor Parameters and Compute.}
Each evolution predictor is a linear map $W_p^{(l)}\in\mathbb{R}^{d\times N}$. The total predictor size is therefore $LdN$ parameters. During route recording, \method~adds one predictor projection per MoE layer, or $2LdN$ extra FLOPs per generated token. Table~\ref{tab:model_config} reports the resulting model-specific values. Under the gated-MLP expert form with gate, up, and down projections, the Feed-Forward Network (FFN) ratio compares this route-recording forward overhead with the active expert computation, approximated as $6L(k+s)dm$ FLOPs per token. For Moonlight-16B-A3B, $s=2$ counts shared experts, while the cached routed indices still use $k=6$.

\paragraph{Cache Cost.}
We store replay expert indices as int32, so the index cache adds $4Lk$ bytes per generated token across all MoE layers and scales with the rollout length $T$. For predictor learning, \method~additionally caches route-recording features $(h_{\text{old}},p_{\text{old}})$ for a fixed feature-cache length $T_c$, with BF16 hidden features and FP32 router logits, contributing $T_cL(2d+4N)$ bytes per response. We use $T_c=2\mathrm{K}$ for Qwen3-30B-A3B-Base, whose maximum generation length is $16\mathrm{K}$. For Moonlight-16B-A3B and OLMoE-1B-7B the entire generation is already $1\mathrm{K}$ tokens, so no downsampling is applied and $T_c=T=1\mathrm{K}$. This separation keeps full expert indices cheap while bounding the feature cache under long contexts.

\section{Additional Results with \methodthree}
\label{app:additional_main_results}

\paragraph{\methodthree~Settings.}
Table~\ref{tab:qwen_r3_comparison} compares \methodthree~with rollout routing replay on Qwen3-30B-A3B-Base. Both methods use the same data and hyperparameters. Unlike the main \method~experiments, this comparison uses learning-rate multipliers of $\{10^{2},\,10^{1},\,10^{1}\}$ for off-2, off-4, and off-8 runs, respectively.

\begin{table}[!htbp]
\centering
\scriptsize
\setlength{\tabcolsep}{6pt}
\renewcommand{\arraystretch}{1.08}
\resizebox{0.8\linewidth}{!}{%
\begin{tabular}{llccccc}
\toprule
\rowcolor{TableHeader}
\textbf{Policy} & \textbf{Method}
& \begin{tabular}[c]{@{}c@{}}\textbf{AIME24}\\\textnormal{\scriptsize(Avg@32)}\end{tabular}
& \begin{tabular}[c]{@{}c@{}}\textbf{AIME25}\\\textnormal{\scriptsize(Avg@32)}\end{tabular}
& \begin{tabular}[c]{@{}c@{}}\textbf{AMC23}\\\textnormal{\scriptsize(Avg@16)}\end{tabular}
& \begin{tabular}[c]{@{}c@{}}\textbf{HMMT25}\\\textnormal{\scriptsize(Avg@16)}\end{tabular}
& \textbf{Average} \\
\midrule
& GRPO + \rthree~ & 42.36 & 28.05 & 84.38 & 13.75 & 42.14 \\
\rowcolor{TableHighlight}
\cellcolor{white}\multirow{-2}{*}{Off-2}
& \textbf{GRPO + \methodthree} & \textbf{44.48} & \textbf{29.69} & \textbf{87.97} & \textbf{16.67} & \textbf{44.70} \\
\midrule
& GRPO + \rthree~ & 41.22 & 28.92 & \textbf{85.21} & 15.56 & 42.73 \\
\rowcolor{TableHighlight}
\cellcolor{white}\multirow{-2}{*}{Off-4}
& \textbf{GRPO + \methodthree} & \textbf{47.19} & \textbf{32.81} & 83.28 & \textbf{20.83} & \textbf{46.03} \\
\midrule
& GRPO + \rthree~ & 36.18 & 24.20 & 79.22 & \textbf{11.39} & 37.75 \\
\rowcolor{TableHighlight}
\cellcolor{white}\multirow{-2}{*}{Off-8}
& \textbf{GRPO + \methodthree} & \textbf{40.42} & \textbf{28.33} & \textbf{81.25} & 10.00 & \textbf{40.00} \\
\bottomrule
\end{tabular}%
}
\vspace{8pt}
\caption{\textbf{Additional \methodthree~results on Qwen3-30B-A3B-Base.} \rthree~denotes rollout routing replay for simplicity. Bold values mark the better accuracy in each metric column within an off-policy strength, and shaded rows mark \methodthree. Results are averaged over 3 different seeds.}
\label{tab:qwen_r3_comparison}
\end{table}

\paragraph{\methodthree~Results.}
Across all three off-policy settings, \methodthree~improves the average competition-math accuracy over \rthree~by $2.56\%$, $3.30\%$, and $2.25\%$ points under off-2, off-4, and off-8, respectively. The gains are consistent on AIME24 and AIME25 in every setting, and on AMC23 and HMMT25 in most settings. The only exceptions are AMC23 under off-4 and HMMT25 under off-8, where \rthree~is higher by $1.93\%$ and $1.39\%$ points. The largest average gain appears under off-4, suggesting that the learned predictive bias in \methodthree~is especially useful when replayed rollouts become stale.

\section{Additional Experimental Analysis}
\label{app:additional_exp_analysis}

We provide additional ablations and diagnostics to examine whether the predictive component of \method~is robust to design choices and whether its route prediction remains reliable across update horizons. All ablations are conducted on Qwen3-30B-A3B-Base.

\paragraph{Predictive Objective Ablation.}
We compare the default \method~predictive loss $\mathcal{L}_{\methodmath}$ with a delta-matching alternative $\mathcal{L}_{\methodmath}^{\Delta}$. The delta-matching variant directly supervises the router-logit residual:
\begin{align*}
\mathcal{L}_{\methodmath}^{\Delta}
= \sum_{l=1}^{L}\mathbb{E}_t\!\left[
D_\mathrm{KL}\!\left(
\mathrm{Softmax}(\langle \Delta p_t^{(l)} \rangle)
\,\big\|\,
\mathrm{Softmax}(b_t^{(l)})
\right)
\right],
\end{align*}
where $\Delta p_t^{(l)} := p_t^{(l)} - p_{\text{old},t}^{(l)}$ and $b_t^{(l)} := h_{\text{old},t}^{(l)} W_p^{(l)}$. As shown in Figure~\ref{fig:loss_type}, both objectives produce similar mean-advantage trajectories and comparable AIME24 improvement. The standard KL objective achieves slightly better final accuracy, so we use it in the main experiments.

\paragraph{Learning-Rate Multiplier Ablation.}
We further sweep the predictor learning-rate multiplier $\alpha$ under off-2. Figure~\ref{fig:abl_alpha} compares $\alpha=1\times10^4$ and $\alpha=1\times10^6$. Both settings yield stable optimization and steadily improve AIME24 accuracy, indicating that \method~does not rely on a narrowly tuned predictor learning rate. The larger multiplier learns faster in early training and reaches comparable final accuracy, while $\alpha=10^4$ gives a slightly smoother mean-advantage trajectory. We therefore use the default multiplier setting in the main experiments.

\begin{figure}[!htbp]
    \centering
    \begin{subfigure}[t]{0.49\linewidth}
        \centering
        \includegraphics[width=\linewidth]{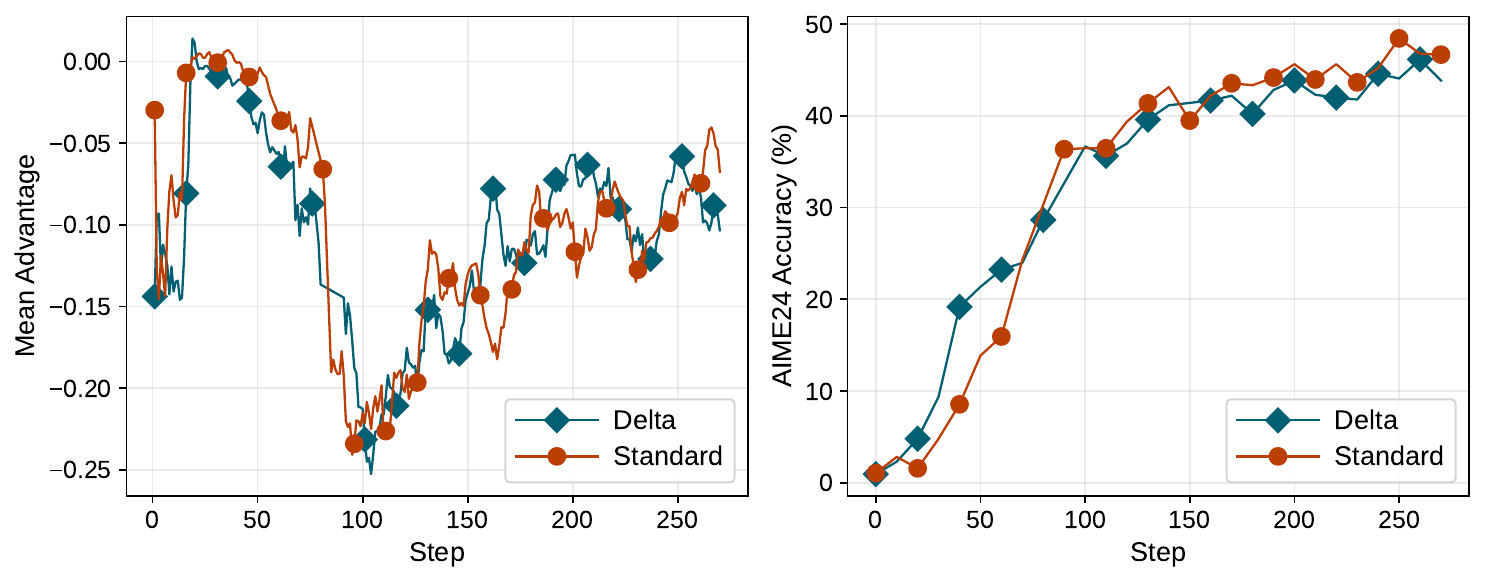}
        \caption{Predictive loss form.}
        \label{fig:loss_type}
    \end{subfigure}
    \hfill
    \begin{subfigure}[t]{0.49\linewidth}
        \centering
        \includegraphics[width=\linewidth]{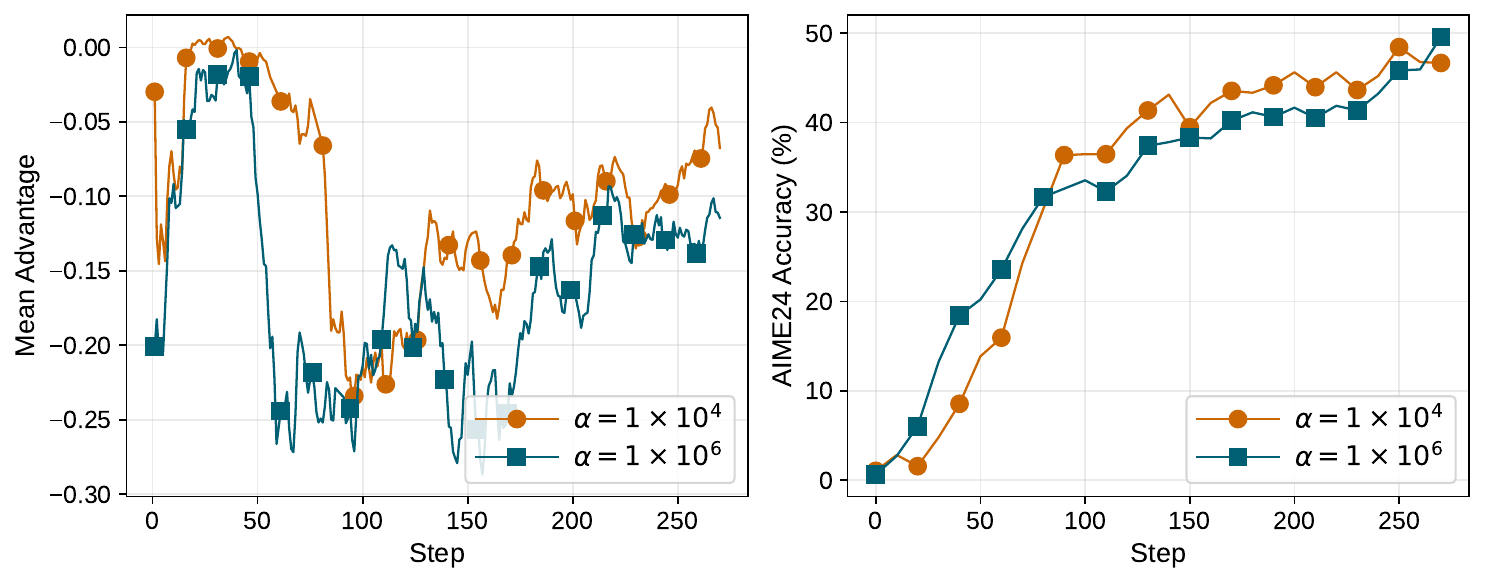}
        \caption{Predictor learning-rate multiplier.}
        \label{fig:abl_alpha}
    \end{subfigure}
    \caption{\textbf{Additional predictive-routing ablations on Qwen3-30B-A3B-Base under off-2.} The panels compare predictive loss forms and predictor learning-rate multipliers, reporting the batch-averaged clipped surrogate integrand and AIME24 accuracy over training steps.}
    \label{fig:app_ablation_summary}
\end{figure}

\paragraph{Layer-Wise Top-$k$ Accuracy.}
\method~is designed to predict short-horizon router evolution rather than perform long-range route forecasting. We therefore measure top-$k$ agreement for each MoE layer and each within-batch mini-step. Figure~\ref{fig:ana_horizon} shows that agreement remains consistently high across off-2, off-4, and off-8, with most layers staying above $98\%$. The curves for different mini-steps largely overlap, suggesting that the prediction remains reliable over multiple reuse updates. The dominant variation comes from layer depth, where later layers show slightly lower agreement, but the values remain high enough for the predictor to provide useful replay supports.

\begin{figure}[!htbp]
\centering
\includegraphics[width=1.0\linewidth]{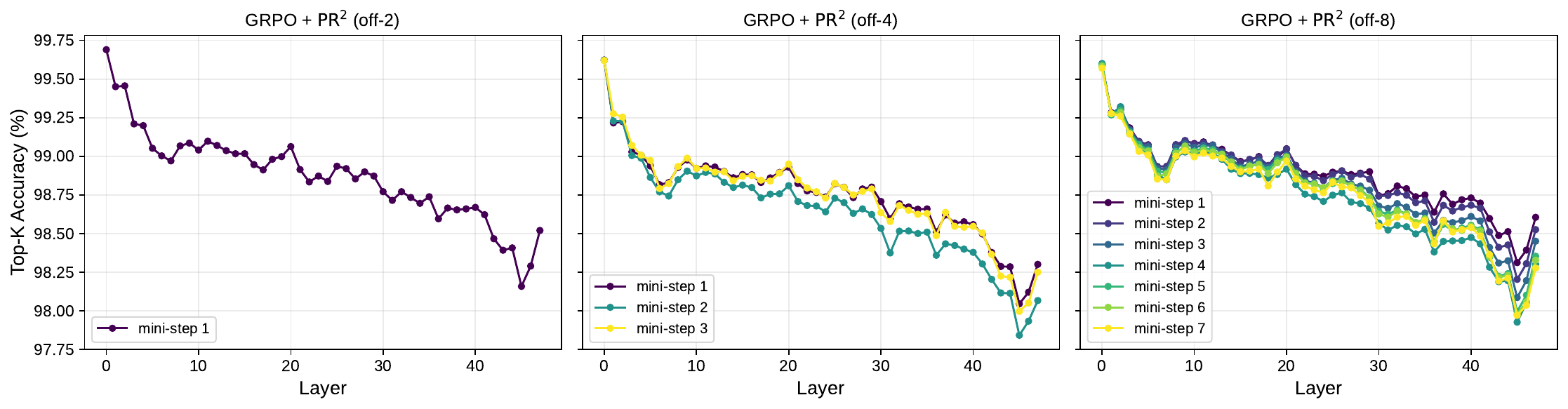}
\vspace{-4pt}
\caption{\textbf{Layer-wise top-$k$ accuracy across mini-steps on Qwen3-30B-A3B-Base.} The horizontal axis indexes MoE layers, and each curve reports the top-$k$ agreement at a different mini-step.}
\label{fig:ana_horizon}
\end{figure}

Overall, Figures~\ref{fig:app_ablation_summary} and~\ref{fig:ana_horizon} show that \method~is robust to the predictive objective form and predictor learning-rate multiplier, while maintaining high route-prediction accuracy across layers and repeated rollout reuse.

\end{document}